\documentclass{article}
\usepackage[accepted]{icml2024}
\usepackage{hyperref}
\hypersetup{
    colorlinks=true,
    citecolor=blue,
    linkcolor=blue,
    filecolor=blue,
    urlcolor=blue
}
\usepackage{minitoc}

\usepackage{microtype}
\usepackage{graphicx}
\usepackage{subfigure}
\usepackage{booktabs} % for professional tables
\usepackage{xurl}
\usepackage{amssymb}
\usepackage{xcolor}
\usepackage{amsmath}
\usepackage{amssymb}
\usepackage{mathtools}
\usepackage{amsthm}
\usepackage{thmtools}
\usepackage{multirow}
\usepackage{multicol}
\usepackage{enumitem}
\usepackage[capitalize,noabbrev]{cleveref}
\usepackage{adjustbox}
\usepackage{threeparttable}
\usepackage{bbding}
\usepackage{colortbl}  %彩色表格需要加载的宏包
\usepackage{xcolor}
\usepackage{array}   %对表列和表格线的设置需要用到array宏包
\usepackage{wrapfig}
\theoremstyle{plain}

\theoremstyle{definition}

\theoremstyle{remark}

\usepackage{amsmath,amsfonts,bm}

\def\1{\bm{1}}

\DeclareMathAlphabet{\mathsfit}{\encodingdefault}{\sfdefault}{m}{sl}
\SetMathAlphabet{\mathsfit}{bold}{\encodingdefault}{\sfdefault}{bx}{n}

\newcommand{\E}{\mathbb{E}}

\usepackage{listings}

\usepackage{lipsum}

\definecolor{myred-color}{RGB}{238, 117, 120}
\definecolor{myblue-color}{RGB}{102, 153, 204}

\definecolor{dkgreen}{rgb}{0,0.6,0}
\definecolor{gray}{rgb}{0.5,0.5,0.5}
\definecolor{mauve}{rgb}{0.58,0,0.82}
\usepackage{comment}

\newcommand{\ours}{\textit{Constrained Gradient-based Policy Optimization}}
\newcommand{\abe}{\textit{Advantage-based Estimation}}
\newcommand{\gbe}{\textit{Gradient-based Estimation}}

\definecolor{revision}{RGB}{255,102,0}

\icmltitlerunning{Safe Reinforcement Learning using Finite-Horizon Gradient-based Estimation}
\begin{document}
% \doparttoc
\twocolumn[
\icmltitle{Safe Reinforcement Learning using Finite-Horizon Gradient-based Estimation}
% \icmlsetsymbol{corr}{*}

\begin{icmlauthorlist}
\icmlauthor{Juntao Dai}{sch1,sch2}
\icmlauthor{Yaodong Yang}{sch3}
\icmlauthor{Qian Zheng}{sch1,sch2}
\icmlauthor{Gang Pan}{sch1,sch2}
\end{icmlauthorlist}

\icmlaffiliation{sch1}{College of Computer Science and Technology, Zhejiang University, Hangzhou, China}
\icmlaffiliation{sch2}{The State Key Lab of Brain-Machine Intelligence, Zhejiang University, Hangzhou, China}
\icmlaffiliation{sch3}{Center for AI Safety and Governance, Peking University, Beijing, China}

\icmlcorrespondingauthor{Qian Zheng}{qianzheng@zju.edu.cn}
\icmlcorrespondingauthor{Gang Pan}{gpan@zju.edu.cn}

\icmlkeywords{Machine Learning, ICML}

\vskip 0.3in
]
\printAffiliationsAndNotice{}
\doparttoc
\faketableofcontents

\begin{abstract}
A key aspect of Safe Reinforcement Learning (Safe RL) involves estimating the constraint condition for the next policy, which is crucial for guiding the optimization of safe policy updates.
However, the existing \textit{Advantage-based Estimation} (ABE) method relies on the infinite-horizon discounted advantage function.
This dependence leads to catastrophic errors in finite-horizon scenarios with non-discounted constraints, resulting in safety-violation updates.
In response, we propose the first estimation method for finite-horizon non-discounted constraints in deep Safe RL, termed \textit{Gradient-based Estimation} (GBE), which relies on the analytic gradient derived along trajectories.
Our theoretical and empirical analyses demonstrate that GBE can effectively estimate constraint changes over a finite horizon.
Constructing a surrogate optimization problem with GBE, we developed a novel Safe RL algorithm called \textit{Constrained Gradient-based Policy Optimization} (CGPO).
CGPO identifies feasible optimal policies by iteratively resolving sub-problems within trust regions.
Our empirical results reveal that CGPO, unlike baseline algorithms, successfully estimates the constraint functions of subsequent policies, thereby ensuring the efficiency and feasibility of each update.

\end{abstract}

\section{Introduction}
Reinforcement Learning (RL) \citep{sutton2018reinforcement} stands as a powerful paradigm in artificial intelligence.
Over the past few years, RL has achieved notable success across various challenging tasks, such as video games~\citep{eldahshan2022deep}, robotic control~\citep{okamura2000overview,singh2022reinforcement}, Go~\citep{silver2016mastering,silver2017mastering}, and the training of large language models~\citep{ouyang2022training,rafailov2023direct}.
Recently, there has been a growing emphasis on prioritizing the safety of policy learning.
This shift is driven by the critical need for safety in real-world applications, such as autonomous driving~\citep{muhammad2020deep} and service robots~\citep{bogue2017robots}.
In response, Safe RL~\citep{garcia2015comprehensive} has emerged as a related paradigm, aiming to provide reliable and robust policy learning in the face of complex and dynamic environments.
The Constrained Markov Decision Process (CMDP)~\citep{altman1999constrained} stands as a foundational framework of Safe RL, augmenting the traditional Markov Decision Process (MDP)~\citep{puterman1990markov} by incorporating constraints.

\begin{figure}[t]
    \centering
    \includegraphics[width=0.98\columnwidth]{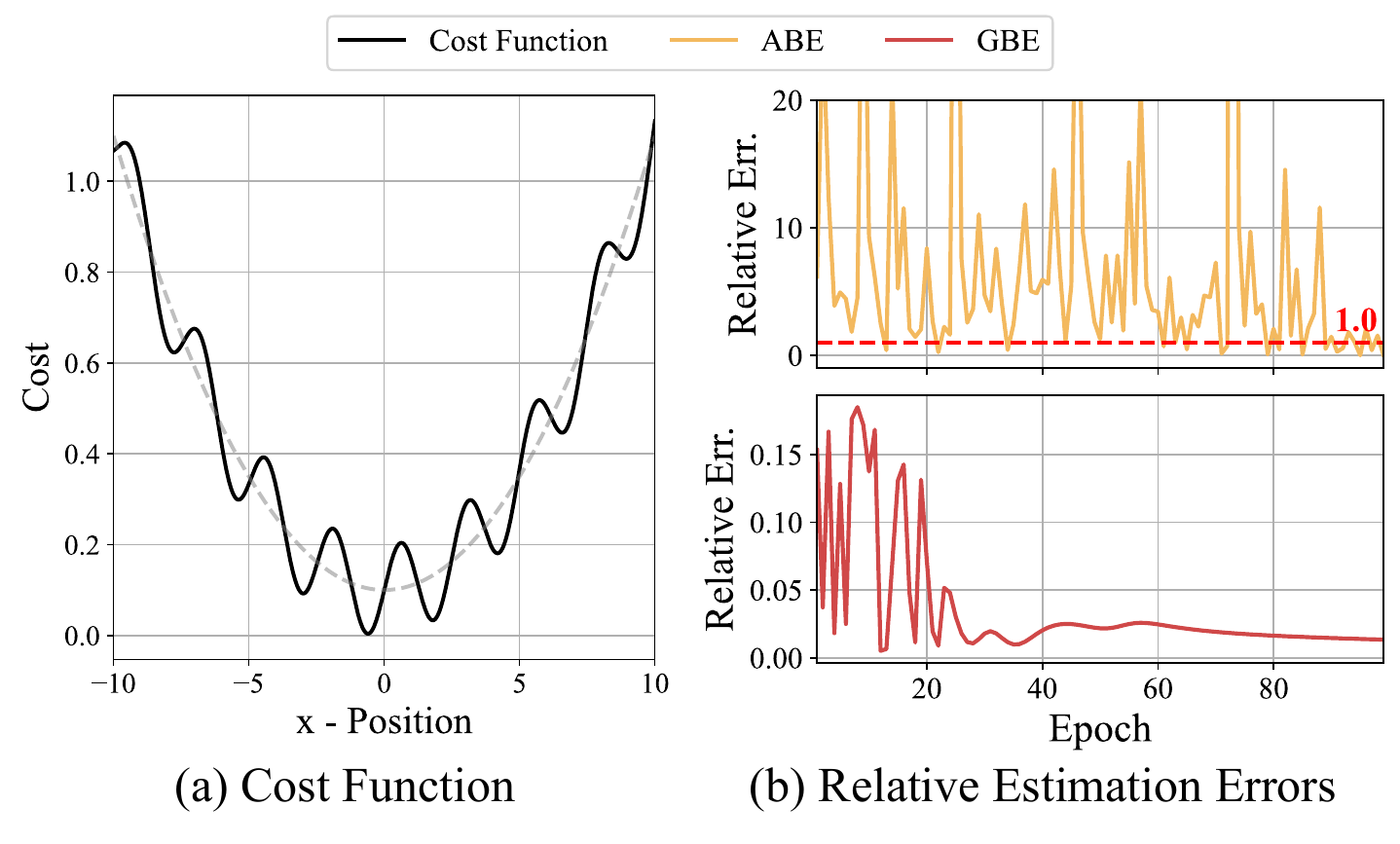}
    \vspace{-0.5em}
    \caption{\textbf{\textit{Advantage-based Estimation} fails even in simple environments under finite-horizon constraints.} (a) The cost obtained by the agent while traversing along the x-axis, namely, $c_t = c(x_t)$. (b) Relative errors in the estimation of changes in the finite-horizon cumulative constraint (i.e., $\sum_{t=1}^\top c_t \leq b$). The ABE method generates relative errors even greater than $1.0$, showing completely incorrect estimations. Refer to Appendix~\ref{app:simple_env} for more details.}
    \label{fig:simple-exp}
\end{figure}

In most Safe RL benchmarks~\citep{ray2019benchmarking,gronauer2022bullet,ji2023safety}, the common practical constraints are represented as non-discounted sums over finite trajectories, subject to scalar thresholds.
However, existing deep Safe RL algorithms uniformly apply an infinite-horizon approach, the \abe{} (ABE)~\citep{achiam2017constrained}, across all types of constraints.
This fundamental discrepancy hinders these algorithms' ability to accurately predict the constraint values for subsequent policies, thereby misleading the optimization direction in constrained problems.
Our straightforward experiment illustrates this issue in Figure~\ref{fig:simple-exp}.
Despite the simplicity of the task, the ABE method generates relative errors exceeding 1.0.

To bridge the gap in estimating finite-horizon constraints, we introduce a new estimation methodology, \gbe{} (GBE).
Unlike previous methods, GBE avoids reliance on the infinite-horizon assumption, instead leveraging first-order gradients derived along finite trajectories~\citep{mohamed2020monte}.
This approach facilitates precise estimation of constraints, particularly where a non-discounted cumulative sum over a finite horizon is compared against a scalar threshold.
Utilizing the GBE, we construct the constrained surrogate problem and develop a novel Safe RL algorithm, called \ours{} (CGPO).
Additionally, informed by an error analysis of GBE, we implement a trust region approach in the parameter space to mitigate errors and further ensure the feasibility of policy updates.

Our contributions are threefold:
\textbf{(1)} We introduce the GBE method, designed to estimate finite-horizon non-discounted constraints in Safe RL tasks, along with relevant theoretical analysis.
\textbf{(2)} We propose the CGPO algorithm. To our best knowledge, CGPO is the first deep Safe RL algorithm that can effectively address tasks with finite-horizon constraints.
Based on precise estimations of the GBE method, CGPO ensures the efficiency and feasibility of each update.
Our analysis includes worst-case scenarios and introduces an adaptive trust region radius strategy for improved performance.
\textbf{(3)} We develop a series of Safe RL tasks with differentiable dynamics to test our algorithm. Comparative evaluations demonstrate our algorithm's superior update efficiency and constraint satisfaction against baselines.

\section{Related Works}

\textbf{Differentiable RL.}
In many RL applications, the knowledge of the underlying transition function facilitates policy optimization through analytical gradients, a method known as Differentiable RL~\citep{mohamed2020monte,jaisson2022deep}.
The development of differentiable simulators~\citep{degrave2019differentiable,werling2021fast,xian2023fluidlab}, which represent systems as differentiable computational graphs, significantly advance this research area.
Some well-known differentiable RL algorithms include BPTT~\citep{bptt}, POSD~\citep{mora2021pods}, and SHAC~\citep{Jie2022shac}.
Moreover, differentiable RL techniques extend to model-based algorithms that access analytical gradients from World Models~\citep{clavera2020model,as2022constrained,pmlr-v202-parmas23a}.
Our algorithm, consistent with other Differentiable RL methods, depends on either a differentiable simulator or the World Model method for modeling task environments. Considering the swift advancements in related techniques, we view this requirement optimistically. For further discussion, see Section~\ref{app:limitation}.

\textbf{Safe RL.}
The work closely related to ours focuses on on-policy deep Safe RL algorithms, which are categorized into Lagrangian and Convex Optimization methods~\citep{xu2022trustworthy}.
Lagrangian methods, such as PDO~\citep{chow2018risk} and CPPO-PID~\citep{stooke2020responsive}, alternate between addressing the primal and the dual problems.
The latest APPO~\citep{dai2023augmented} mitigates the oscillatory issue by augmenting a simple quadratic term.
Convex Optimization methods, such as CPO~\citep{achiam2017constrained} and CVPO~\citep{liu2022constrained}, optimize based on the solution to the primal problem.
The latest method, CUP~\citep{yang2022constrained}, divides updates into two steps: updating in the steepest direction first, then mapping into the feasible domain if constraints are violated.
Above all methods use estimates of objective and constraint functions to formulate surrogate problems~\citep{achiam2017constrained}, which we refer to as \abe{}.
Thus, the accuracy of estimates impacts the efficiency and feasibility of updates.

However, research on finite-horizon constraints has been limited to non-deep Safe RL contexts~\citep{kalagarla2021sample,guin2023policy}.
Despite their prevalence in the Safe RL Benchmark, aforementioned deep Safe RL algorithms treat these constraints as if they were infinite-horizon (such as Safety-Gym~\citep{ray2019benchmarking}, Bullet-Safety-Gym~\citep{gronauer2022bullet}, Safety-Gymnasium~\citep{ji2023safety}, and OmniSafe~\citep{ji2023omnisafe}).
This leads to poor constraint satisfaction within these benchmarks.

\section{Preliminaries}
\label{sec:prelim}

\subsection{Constrained Markov Decision Process}
RL is typically framed as a Markov Decision Process (MDP)~\citep{puterman2014markov}, denoted $\mathcal{M}\triangleq\left\langle \mathcal{S}, \mathcal{A}, r, {P}, \mu_0 \right\rangle$, which encompasses the state space $\mathcal{S}$, the action space $\mathcal{A}$, a reward function $r$, the transition probability ${P}$, and the initial state distribution $\mu_0$.
A stationary policy $\pi$ represents a probability distribution defining the likelihood of taking action $a$ in state $s$.
The set $\Pi$ symbolizes the collection of all such stationary policies.
The primary objective of RL is to optimize the performance metric typically formulated as the total expected return over a finite horizon $T$, namely, $\mathcal{J}_R(\pi) = \E_{\tau\sim\pi}\left[\sum^{T-1}_{t=0}r(s_t, a_t)\right]$.

Generally, Safe RL is formulated as a Constrained MDP (CMDP)~\citep{altman1999constrained}, $\mathcal{M} \cup \mathcal{C}$, augmenting the standard MDP $\mathcal{M}$ with an additional set of constraints $\mathcal{C}$.
This constraint set $\mathcal{C}=\left\{\left(c_{i}, b_{i}\right)\right\}_{i=1}^{m}$ consists of pairs of cost functions $c_{i}$ and corresponding thresholds $b_{i}$.
The constraint function is defined as the cumulative cost over the horizon $T$, $\mathcal{J}_{C_{i}}(\pi)=\mathbb{E}_{\tau\sim\pi}\left[\sum_{t=0}^{T-1} c_{i}\left(s_{t}, a_{t}\right)\right]$, and the feasible policy set is $\Pi_\mathcal{C}=\mathop{\cap}^m_{i=1}\{\pi\in\Pi\mid\mathcal{J}_{C_{i}}(\pi)\leq b_i\}$.
The objective of Safe RL is to find the optimal feasible policy, $\pi^{\star}= \mathop{\arg \max}_{\pi \in \Pi_{\mathcal{C}}}\mathcal{J}_R(\pi)$.

\subsection{Advantage-based Estimation Method}\label{sec:abe_anlysis}

By introducing the discount factor $\gamma$, we define the objective function for an infinite-horizon scenario as $\mathcal{J}^\gamma_R(\pi) = \mathbb{E}_{\tau\sim\pi}\left[\sum^{\infty}_{t=0}\gamma^t r(s_t, a_t)\right]$ and the constraint function as $\mathcal{J}^\gamma_{C_{i}}(\pi) = \mathbb{E}_{\tau\sim\pi}\left[\sum_{t=0}^{\infty} \gamma^t c_{i}(s_{t}, a_{t})\right]$.
We express the infinite-horizon value function as $V^\gamma_\pi(s)= \mathbb{E}_{\tau\sim\pi}\left[\sum^{\infty}_{t=0}\gamma^t r_t\mid s_0=s\right]$ and the state-action value function as $Q^\gamma_\pi(s,a)= \mathbb{E}_{\tau\sim\pi}\left[\sum^{\infty}_{t=0}\gamma^t r_t\mid s_0=s,a_0=a\right]$. The infinite-horizon advantage function is $A^\gamma_\pi(s,a)=Q^\gamma_\pi(s,a) - V^\gamma_\pi(s)$.
The discounted future state distribution $d^\gamma_\pi$ is denoted as $d^\gamma_\pi(s)=(1-\gamma)\sum^\infty_{t=0}\gamma^tP(s_t=s\mid\pi)$. Then, the difference in some metrics between two policies $\pi$, $\pi'$ can be derived as~\citep{kakade2002approximately}:
\begin{equation}\label{eq:policy_diff}
    \mathcal{J}^\gamma_f(\pi')-\mathcal{J}^\gamma_f(\pi) = \frac{1}{1-\gamma}\E_{\substack{s\sim d^\gamma_{\pi'} \\ a\sim{\pi'}}}\Big[A^\gamma_\pi(s,a)\Big],
\end{equation}
where $\mathcal{J}^\gamma_f$ represents infinite-horizon $\mathcal{J}^\gamma_R$, $\mathcal{J}^\gamma_C$.
Equation~\eqref{eq:policy_diff} is difficult to estimate since $s\sim d^\gamma_{\pi'}$ is unknown. \citet{achiam2017constrained} approximates $s\sim d^\gamma_{\pi'}$ with $s\sim d^\gamma_{\pi}$ and employs importance sampling techniques to derive the following estimation:
\begin{equation}\label{eq:abm}
    \bar{\mathcal{J}}^\gamma_f(\pi') \triangleq \mathcal{J}^\gamma_f(\pi) + \frac{1}{1-\gamma}\E_{\substack{s\sim d^\gamma_{\pi} \\ a\sim{\pi}}} \Big[\frac{\pi'(s,a)}{\pi(s,a)}A^\pi(s,a)\Big],
\end{equation}
where $d^\pi$ and $A^\pi(s,a)$ are both defined in the infinite-horizon format and it requires that $\gamma\neq1$.
Thus, algorithms based on Equation~\eqref{eq:abm} have to treat all constraints as if they were infinite-horizon.

\subsection{Policy Optimization with Differentiable Dynamics}

Differentiable simulators~\citep{freeman2021brax} represent physical rules using a differentiable computational graph, $\bm{s}_{t+1}=\mathcal F(\bm{s}_t,\bm{a}_t)$, allowing them to participate in gradient back-propagation process~\citep{mohamed2020monte}.
The fundamental approach for policy optimization with differentiable physical dynamics is Back-propagation Through Time (BPTT)~\citep{mozer2013focused}.
Multiple trajectories $\{\tau_i\}^N_{i=1}$ are collected from and then derivation is performed along the trajectory.
Then, the loss function is
\begin{equation}
    \mathcal{L}^\text{BPTT} = -\frac{1}{NT}\sum^N_{i=1}\sum^{T-1}_{t=0}r(\bm{s}^i_t,\bm{a}^i_t).
\end{equation}

The Short-Horizon Actor-Critic (SHAC) approach \citep{Jie2022shac} modifies BPTT by ultilizing the value function $V$ to segment trajectories into sub-windows of length $h$:
\begin{equation}
    \mathcal{L}^\text{SHAC} = -
    \frac{1}{Nh}\sum_{i = 1}^{N}\Big[\sum_{t = t_0}^{t_0 + h-1}\gamma^{t - t_0}r(\bm{s}^i_t,\bm{a}^i_t) + \gamma^h V(\mathbf{s}^i_{t_0 + h})\Big].
\end{equation}
This change limits the maximum length for gradient back-propagation to $h$, thereby mitigating the issues of gradient vanishing/exploding and making the training more stable.

\section{Constrained Surrogate Problem using Gradient-based Estimation}
\label{sec:opt}

Differentiable environments represent physical dynamics as differentiable computational graphs, allowing first-order gradients of objective and constraint functions to be derived via gradient back-propagation along trajectories.
In this section, we propose a new estimation method based on this feature, called \gbe{} (GBE).
It facilitates more accurate approximations for the next police's objective and constraint functions, which leads to a new constrained surrogate problem for solving Safe RL.

\subsection{Gradient-based Estimation for Objective and Constraint Functions}
\label{sec:bound}
Consider a parameterized policy $\pi_\theta$ within a parameter space $ \Theta $, e.g., represented by a neural network.
In the context of differentiable environments, both the objective function $\mathcal{J}_R(\bm\theta)$ and the constraint function $ \mathcal{J}_C(\bm{\theta}) $ can be regarded as differentiable over $\Theta$.
Consequently, we can compute the first-order gradients of them along the trajectories, denoted as {$\nabla_{\bm\theta}\mathcal{J}_R(\bm{\theta})$ and $\nabla_{\bm\theta}\mathcal{J}_C(\bm{\theta})$}, through back-propagation.
For simplicity and generality, we will focus on a single constraint scenario, though the method applies to multiple constraints via corresponding matrix operations.

For a minor update $\bm \delta$ in the policy parameter space, transitioning from $ \bm{\theta}_0 $ to $ \bm{\theta}_0 + \bm{\delta} $, consider the function $\mathcal{J}_f$, which encapsulates both the objective function $ \mathcal J_R $ and the constraint function $ \mathcal J_C$. {We perform a first-order Taylor expansion of $\mathcal{J}_f$ at $\bm{\theta}_0$ to obtain:
\begin{equation}
    \label{eq:taylor}
    \begin{aligned}
    \mathcal J_f(\bm{\theta}_0+\bm{\delta}) = \mathcal J_f (\bm{\theta}_0) &+ \bm{\delta}^\top \nabla_{\bm{\theta}} \mathcal J_f(\bm{\theta}_0) \\&+ \frac{1}{2} \bm{\delta}^\top\nabla_{\bm{\theta}}^2 \mathcal J_f(\bm{\theta}_0+t\bm{\delta})\bm{\delta}.
    \end{aligned}
\end{equation}}
Wherein, $t\in(0,1)$ and $\frac{1}{2} \bm{\delta}^\top\nabla_{\bm{\theta}}^2 \mathcal J_f(\bm{\theta}_0+t\bm{\delta})\bm{\delta}$ represents the Peano remainder term, which is $ o(\|\bm{\delta}\|) $.
When $ \bm{\delta} $ is sufficiently small, the remainder term becomes negligible.
Therefore, we propose the \textit{Gradient-based Estimation} method to estimate the values of the objective and constraint functions after a minor update $\bm\delta$ from $\bm{\theta}_0$ as follows:
\begin{equation}
    \hat{\mathcal J}_R(\bm{\theta}_0+\bm{\delta}) \triangleq \mathcal J_R (\bm{\theta}_0) + \bm{\delta}^\top \nabla_{\bm{\theta}} \mathcal J_R(\bm{\theta}_0),
    \label{eq:estimation_obj}
\end{equation}
\begin{equation}
    \hat{\mathcal J}_C(\bm{\theta}_0+\bm{\delta}) \triangleq \mathcal J_C (\bm{\theta}_0) + \bm{\delta}^\top \nabla_{\bm{\theta}} \mathcal J_C(\bm{\theta}_0).
    \label{eq:estimation_cons}
\end{equation}
Regarding the error analysis for these estimates, the following lemma states:
\begin{restatable}{lemma}{TheEstError}
    \label{the:estimation_error}
    Assume $ \bm{\theta}_0 \in \Theta $ and $ \mathcal J_f(\bm{\theta}) $ is twice differentiable in a neighborhood surrounding $ \bm{\theta}_0 $. Let $ \bm{\delta} $ be a small update from $ \bm{\theta}_0 $. If we estimate $\mathcal J_f(\bm{\theta}_0+\bm{\delta})$ as $\hat{\mathcal J}_f(\bm{\theta}_0+\bm{\delta}) = \mathcal J_f (\bm{\theta}_0) + \bm{\delta}^\top \nabla_{\bm{\theta}} \mathcal J_f(\bm{\theta}_0)$ and given that $\epsilon = \max_{t\in(0,1)}\left|\nabla_{\bm{\theta}}^2 \mathcal J_R(\bm{\theta}_0+t\bm{\delta})\right|$, the estimation error can be bounded as:
    \begin{equation}
        \label{eq:estimation_error}
        \left|\hat{\mathcal J}_f(\bm{\theta}_0+\bm{\delta}) - {\mathcal J}_f(\bm{\theta}_0+\bm{\delta})\right|\leq\frac{1}{2} \epsilon \|\bm{\delta}\|^2_2.
    \end{equation}
\end{restatable}
\begin{proof}
The primary source of estimation error is the neglect of higher-order infinitesimal remainders.
Thus, the error is $\left|\hat{\mathcal J}_f(\bm{\theta}_0+\bm{\delta}) - {\mathcal J}_f(\bm{\theta}_0+\bm{\delta})\right| = \frac12 \big|\bm{\delta}^\top\nabla_{\bm{\theta}}^2 \mathcal J_f(\bm{\theta}_0+t\bm{\delta})\bm{\delta}\big|\leq\frac{1}{2} \epsilon \|\bm{\delta}\|^2_2$.
See Appendix \ref{app:first_order_error} for more details.
\end{proof}
This finding indicates that as policy parameters are updated from $\bm{\theta}_0$ to $\bm{\theta}_0 + \bm{\delta}$, the upper bounds of the estimation errors for both the objective and constraint functions are positively correlated with the square of the $L_2$ norm of the update vector $\bm{\delta}$.
Consequently, when employing $\hat{\mathcal J}_R$ and $\hat{\mathcal J}_C$ as surrogate objective and constraint functions, meticulous control over the magnitude of these updates becomes essential.
By carefully managing $\|\bm{\delta}\|^2_2$, we could ensure that the estimation error remains within an acceptable range, thus facilitating both effective performance improvement and precise adherence to constraints.

\subsection{Constrained Surrogate Problem within Trust Region}

Based on the analysis of Theorem~\ref{the:estimation_error}, the idea of controlling the error by managing $\|\bm{\delta}\|^2_2$ naturally aligns with the concept of trust regions~\citep{schulman2015trust, 9334437}.
Given a trust region radius $\hat\delta$, we suppose that updated parameter $\bm{\theta_{k+1}}$ of the $k^\text{th}$ iteration within the trust region $\Theta_k=\left\{ \bm{\theta} \in \Theta \mid \|\bm{\theta} - \bm{\theta}_k\|^2 \leq \hat\delta \right\}$ are credible.

By employing the approximations of the objective function in Equation~\eqref{eq:estimation_obj} and the constraint function in Equation~\eqref{eq:estimation_cons}, We transform the solution of the primal Safe RL problem into an iterative process of solving a series of sub-problems within predefined trust regions.
Given the initial policy parameter $\bm{\theta}_{k}$ of the $k^\text{th}$ iteration, our sub-problem targets finding the next optimal and credible parameter  $\bm{\theta}_{k+1}\in\Theta_k$, which not only maximize the surrogate objective function $\hat{\mathcal J}_R(\bm\theta_{k+1})$ but also conform to the surrogate constraint $\hat{\mathcal J}_C(\bm\theta_{k+1}) \leq b$. Thus, the surrogate sub-problem within a given trust region at $k^{th}$ iteration can be represented as:
\begin{equation}
    \begin{aligned}
        \bm{\theta}_{k+1} = &\arg\max_{\bm{\theta}\in\Theta} \quad  \left(\bm{\theta} - \bm{\theta}_{k}\right)^\top \nabla_{\bm{\theta}} \mathcal J_R(\bm{\theta}_k)                                      \\
        \text{s.t.} \quad                                         & \mathcal J_C(\bm{\theta}_k) + \left(\bm{\theta} - \bm{\theta}_{k}\right)^\top \nabla_{\bm{\theta}} \mathcal J_C(\bm{\theta}_k) \leq b \\
                                                              & \left(\bm{\theta} - \bm{\theta}_{k}\right)^\top\left(\bm{\theta} - \bm{\theta}_{k}\right) \leq \hat\delta.
    \end{aligned}
    \label{eq:sub-problem}
\end{equation}

\begin{figure}[t]
    \centering
    \includegraphics[width=0.95\linewidth]{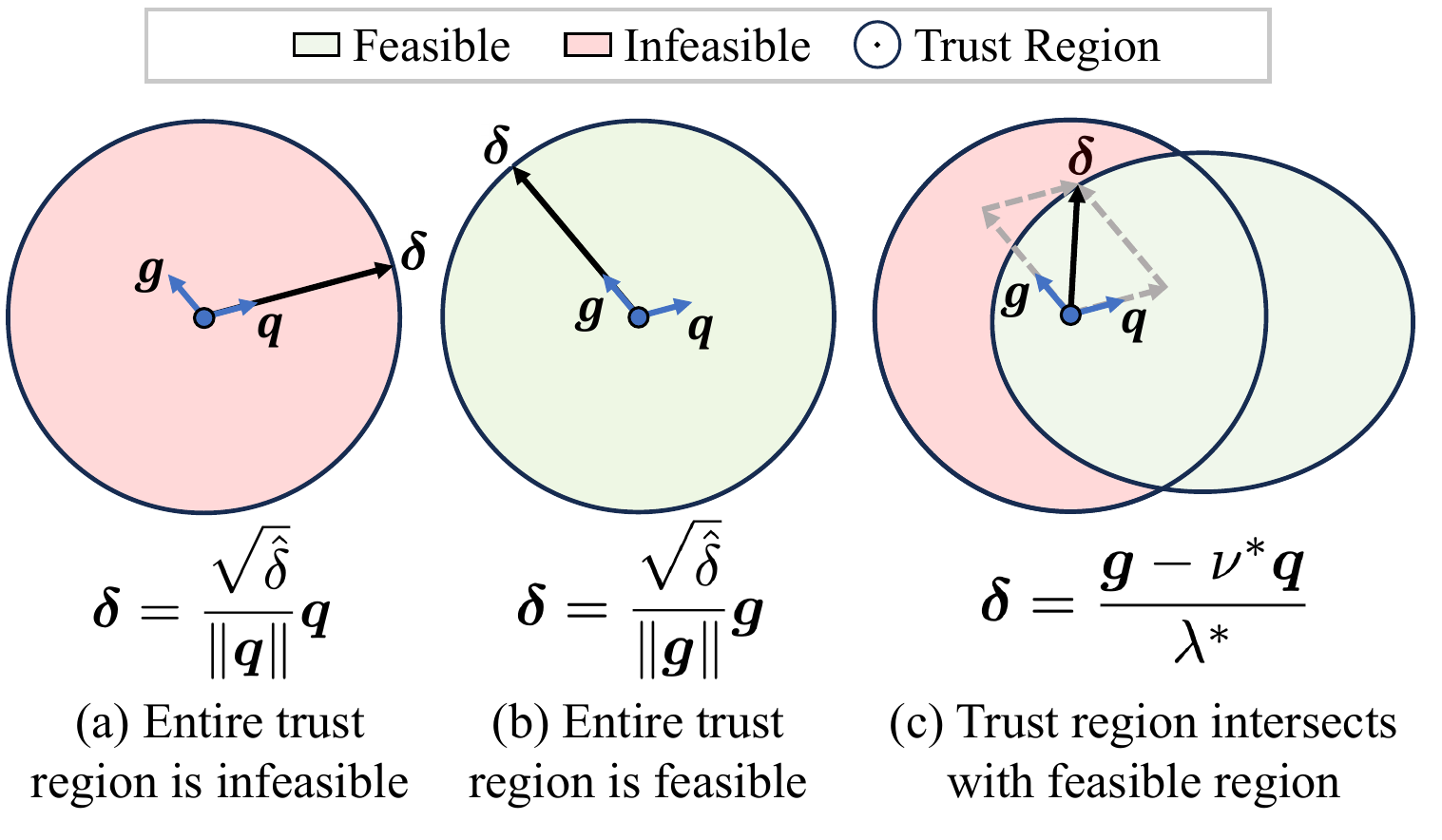}
    \caption{The computational relationship between the policy update $\bm{\delta}$, the gradient of the objective function $\bm{g}$, and the gradient of the constraint function $\bm{g}$ varies in three scenarios.}
    \label{fig:delta}
\end{figure}

\section{Constrained Gradient-based Policy Optimization}
\label{sec:impl}
Based on the constrained surrogate sub-problem in Equation~\eqref{eq:sub-problem}, we develop a novel Safe RL algorithm named \textit{Constrained Gradient-based Policy Optimization} (CGPO).

\subsection{Solution to Surrogate Sub-problem}\label{sec:sol-to-sub-prob}
\textbf{Notations.} Considering the $ k^\text{th} $ iteration within the trust region $ \Theta_k $, we introduce additional notations to make the discussion more concise:
$\bm{g}_k\triangleq\nabla_{\bm{\theta}} \mathcal J_R(\bm{\theta}_k)$,
$\bm{g}_k\triangleq\nabla_{\bm{\theta}} \mathcal J_C(\bm{\theta}_k)$,
$c_k\triangleq\mathcal{J}_C(\bm\theta_k)-b$, and
$\bm\delta\triangleq\bm\theta-\bm\theta_k$.
With these definitions, we rewrite the sub-problem ~\eqref{eq:sub-problem} within the trust region:
\begin{equation}
  \begin{aligned}
    \label{eq:rewrite_problem}
    \max_{\bm{\delta}}~\bm{g}_k^\top\bm{\delta}\quad
    \text{s.t.}~c_k + \bm{q}^\top_k\bm{\delta} \leq 0, \quad \bm{\delta}^\top\bm{\delta} \leq \hat\delta
  \end{aligned}
\end{equation}

Since the sub-problem~\eqref{eq:rewrite_problem} may have no solution, we first discuss the conditions under which this problem is unsolvable. The following theorem holds:
\begin{restatable}[Solvability Conditions for the Sub-problem]{theorem}{CondOfInfea}\label{the:cond_of_sub}
    The sub-problem~\eqref{eq:rewrite_problem} is unsolvable if and only if $c_k^2 / \bm{q}^\top_k\bm{q}_k - \hat\delta > 0$ and $c_k>0$.
\end{restatable}
\begin{proof}
    Consider whether there is an intersection between the trust region and the feasible half-space. For a detailed proof, refer to Appendix~\ref{app:solvability}.
\end{proof}
Through a similar proof, we arrive at the following corollary:
\begin{restatable}{corollary}{AllFea}\label{the:all_feasible}
    $\bm{\theta}$ is deemed feasible for every $\bm{\theta}$ within $\Theta_k$ if and only if $c_k^2 / \bm{q}_k^\top\bm{q}_k - \hat{\delta} > 0$ and $c_k \leq 0$ are both satisfied.
\end{restatable}

Based on Theorem~\ref{the:cond_of_sub} and Corollary~\ref{the:all_feasible}, we solve for $\bm\theta_{k+1}$ of sub-problem~\eqref{eq:sub-problem} in three different scenarios, as illustrated in Figure~\ref{fig:delta}.
\begin{enumerate}[label=(\alph*),leftmargin=*]
    \item If $c^2/{\bm{q}^\top\bm{q}}-\hat\delta>0$ and $c>0$, the entire trust region is infeasible. We update the policy along the direction of the steepest descent of the constraint function, namely, $\bm{\theta}_{k+1} = \bm{\theta}_{k} - \frac{\sqrt{\hat\delta}}{\|\bm{q}_k\|}\bm{q}_k.$
    \item If $c^2/{\bm{q}^\top\bm{q}}-\hat\delta>0$ and $c\leq0$, the trust region lies entirely within the constraint-satisfying half-space. Similarly, we update along the direction of the steepest ascent of the objective function, namely, $\bm{\theta}_{k+1} = \bm{\theta}_{k} + \frac{\sqrt{\hat\delta}}{\|\bm{g}_k\|}\bm{g}_k.$
    \item In other cases, the trust region partially intersects with the feasible region, so we have to solve the constrained quadratic sub-problem~\eqref{eq:rewrite_problem}.
\end{enumerate}
\vspace{-0.5em}
The sub-problem~\eqref{eq:rewrite_problem} is characterized as convex.
When the trust region is partially feasible, the existence of at least one strictly feasible point is guaranteed.
Consequently, strong duality holds as the conditions of Slater's theorem are fulfilled.
By introducing two dual variables $\lambda$ and $\nu$, we construct the dual function of the primal problem as follows:
\begin{equation}
  \begin{aligned}
    L(\bm\delta, \lambda, \nu)
     & = -\bm{g}_k^\top\bm{\delta} + \nu\left(c_k + \bm{q}_k^\top\bm{\delta}\right) + \frac{\lambda}{2}\left(\bm{\delta}^\top\bm{\delta} - \hat\delta\right)
  \end{aligned}
\end{equation}
Since the sub-problem~\eqref{eq:rewrite_problem} satisfies the strong duality condition, any pair of primal and dual optimal points $(\bm\delta^*_k,\lambda^*_k,\nu^*_k)$ must satisfy the KKT conditions, namely,
\begin{align}
    &\nabla_{\bm\delta}L(\bm\delta,\lambda,\nu) = -\bm{g}_k + \nu \bm{q}_k + \lambda\bm\delta = 0,
    \label{eq:main_kkt_1}
    \\
    &\nu\left(c_k + \bm{q}_k^\top\bm{\delta}\right) = 0,
    \\
    &\lambda\left(\bm{\delta}^\top\bm{\delta} - \hat\delta\right) = 0,
    \\
    &c_k+\bm{q}^\top_k\bm\delta\leq0,\quad\bm\delta^\top\bm\delta - \hat\delta \leq0,\quad\nu \geq 0, \quad\lambda \geq 0.
    \label{eq:main_kkt_4}
\end{align}
Note that the optimal dual variables $(\lambda_k^*, \nu_k^*)$ for Equation~\eqref{eq:main_kkt_1}-\eqref{eq:main_kkt_4} can be directly expressed as a function of $\bm{g}_k$ and $\bm{q}_k$. We provide their pseudo-code in Algorithm~\ref{alg:calc_dual_var}. For a detailed solution process of Equation~\eqref{eq:main_kkt_1}-\eqref{eq:main_kkt_4}, refer to Appendix~\ref{app:solution-to-subprob}.
Then, following the derivation from Equation~\eqref{eq:main_kkt_1}, we update the policy by $\bm{\theta}_{k+1} = \bm{\theta}_{k} + \frac{\bm{g}_k-\nu^*_k\bm{q}_k}{\lambda^*_k}$.

So far, the three scenarios for updating policies by solving the surrogate sub-problem~\eqref{eq:sub-problem} have been fully presented.
This forms the core component of our CGPO algorithm.

\subsection{Worst-Case Analysis}\label{sec:conv_analysis}

In the following theorem, we detail the bounds for performance update and constraint violations under the worst-case scenarios after the policy update which results from solving the surrogate sub-problem~\eqref{eq:sub-problem}.
\begin{restatable}[Worst-Case Performance Update and Constraint Violation]{theorem}{WorstCase}
  \label{the:cgpo_wc}
  Suppose $\bm{\theta}_k$, $\bm{\theta}_{k+1}\in\Theta$ are related by Equation~\eqref{eq:sub-problem}. If $\bm{\theta}_k$ is feasible, a lower bound on the policy performance difference between $\bm{\theta}_{k+1}$ and $\bm{\theta}_k$ is
  \begin{equation}
    \label{eq:DCPOPerformance}
    \mathcal J_R (\bm{\theta}_{k+1}) - \mathcal J_R (\bm{\theta}_k) \geq - \frac{1}{2} \epsilon^R_k \hat\delta
  \end{equation}
  where $\epsilon^R_k = \max_{t\in(0,1)}\left|\nabla_{\bm{\theta}}^2 \mathcal J_R(\bm{\theta}_k+t(\bm{\theta}_{k+1}-\bm{\theta}_k))\right|$.

  An upper bound on the constraint objective function of $\bm{\theta}_{k+1}$ is
  \begin{equation}
    \label{eq:DCPOConstraint}
    \mathcal J_C (\bm{\theta}_{k+1}) \leq b + \frac{1}{2}\epsilon^C_k \hat\delta,
  \end{equation}
  where $\epsilon^C_k = \max_{t\in(0,1)}\left|\nabla_{\bm{\theta}}^2 \mathcal J_C(\bm{\theta}_k+t(\bm{\theta}_{k+1}-\bm{\theta}_k))\right|$.
\end{restatable}
\begin{proof}
    See Appendix \ref{app:convergence}.
\end{proof}

Theorem~\ref{the:cgpo_wc} ensures that the application of the update method purposed in Section~\ref{sec:sol-to-sub-prob}, even in the worst-case scenario, will not drastically reduce performance or severely violate constraints with each iteration.
This provides theoretical support for the stability and reliability of the algorithm.

Moreover, according to Theorem~\ref{the:cgpo_wc}, the bounds for performance update and constraint violation are related to $\epsilon^R_k$ and $\epsilon^C_k$, which are challenging to obtain in practice.
Therefore, we indirectly mitigate their impact through an adaptive radius $\hat\delta$.
Based on Equation~\eqref{eq:DCPOPerformance} and \eqref{eq:DCPOConstraint}, we define two new metrics, the reduction ratio $\rho_k$ for the objective $\mathcal J_R$ and the toleration ratio $\zeta_k$ for the constraint, namely,
$\rho_k \triangleq \frac{\mathcal{J}_R(\bm{\theta}_k)-\mathcal{J}_R(\bm{\theta}_{k+1})}{\hat{\mathcal{J}}_R(\bm{\theta}_k)-\hat{\mathcal{J}}_R(\bm{\theta}_{k+1})}$, $\zeta_k \triangleq \frac{\left|b -\mathcal{J}_C(\bm{\theta}_{k+1})\right|}
    {\left|\mathcal{J}_C(\bm{\theta}_{k+1})-\hat{\mathcal{J}}_C(\bm{\theta}_{k+1})\right|}.$

The reduction ratio $\rho_k$ evaluates the consistency between estimated and actual changes in the objective function. The toleration ratio $\zeta_k$ assesses error tolerance in the constraint satisfaction, suggesting more cautious updates as it approaches the threshold.
If $\rho_k$ and $\zeta_k$ fall short of expectations, we adjust the update radius $\hat\delta$ accordingly.
Given thresholds $0 < \eta_1 < \eta_2 < 1$ and update rates $0 < \beta_1 < 1 < \beta_2$ within the $\hat\delta$ range of $[\underline\delta, \overline\delta]$, we derive the following $\hat\delta$ update rule:
\begin{equation}
\label{eq:update_radius}
\hat\delta_{k+1} = \left\{
    \begin{array}{lr}
        \max(\beta_1 \hat\delta_k,~\underline\delta) & \text{if } \rho_k < \eta_1\mid\zeta_k < \eta_1 , \\
        \hat\delta_k & \text{otherwise}, \\
        \min(\beta_2 \hat\delta_k,~\overline\delta) & \text{if } \rho_k \geq \eta_2~\&~\zeta_k\geq\eta_2.
    \end{array}
\right.
\end{equation}
The above adaptive approach can be considered a plugin for CGPO, transforming a challenging-to-tune parameter into several more manageable parameters, thereby enhancing the algorithm's stability and convergence.

\subsection{Practical Implementation} \label{sec:gradient_calc}

Given our focus on employing gradients to solve Safe RL problems, our algorithm can incorporate any differentiable method to compute gradients $\bm{g_k}$ and $\bm{q_k}$.
Various methods are available, broadly classified into two categories: the Zero-order Batch Gradient (ZoBG) method and the First-order Batch Gradient (FoBG) method~\citep{suh2022differentiable}.
In this section, we present an example implementation of CGPO based on a variant of SHAC~\citep{Jie2022shac}, which is a FoBG method. Due to space limitations, we leave the discussion comparing ZoBG and FoBG in the Appendix~\ref{app:zobg}. Our empirical results show that the advantages of FoBG are more suited for the Safe RL field.

\begin{figure}[t]
  \centering
  \includegraphics[width=0.95\linewidth]{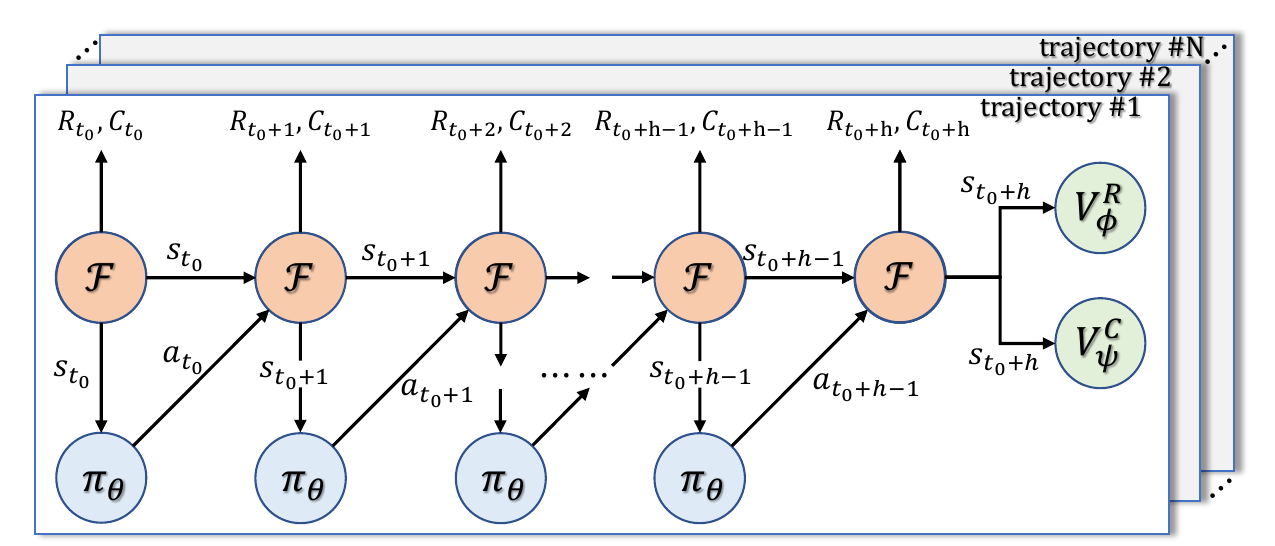}
  \caption{Gradient computation graph for the short-horizon approach. Here, $ \mathcal F(\bm{s}) $ represents the differentiable dynamics of the environment, $ R $ is the reward signal, $ C $ is the cost signal, $ \pi_\theta $ is the parameterized policy, and $ V^R_\phi $ and $ V^C_\phi $ are the value functions for the return and constraint.}
  \label{fig:loss}
\end{figure}

The short-horizon approach from SHAC segments the entire trajectory into sub-windows using the value function, improving differentiability and smoothing the optimization space.
In constrained tasks, we need to compute the objective gradients $\bm{g_k}$ and the constraint gradients $\bm{q_k}$ within the same computational graph.
Consequently, the entire computational graph, as Figure~\ref{fig:loss}, is associated with the corresponding loss of objective and constraint functions:
\begin{equation}
    \mathcal{L}_R(\bm{\theta}) = \frac{1}{N}\sum_{i = 1}^{N}\Big[\sum_{t = t_0}^{t_0 + h-1}r(\mathbf{s}^i_t, \mathbf{a}^i_t) + V^R_\phi(\mathbf{s}^i_{t_0 + h})\Big],
\label{eq:loss_r}
\end{equation}
\begin{equation}
  \label{eq:loss_c}
  \mathcal{L}_C(\bm{\theta}) = \frac{1}{N}\sum_{i = 1}^{N}\Big[\sum_{t = t_0}^{t_0 + h-1}c(\mathbf{s}^i_t, \mathbf{a}^i_t) + V^C_\psi(\mathbf{s}^i_{t_0 + h})\Big].
\end{equation}
Please see Appendix~\ref{app:grad_calc} for detailed calculations.

After updating the policy $\pi_{\bm{\theta}}$, we use these trajectory $\{\tau_i\}_{i=1}^N$ to update critic networks $V^R_\pi$ and $V^C_\psi$.
Considering that the constraint function is an accumulative sum over a finite trajectory, it is necessary for the time step $t$ to be accepted by the critic network to construct value functions $V(s^i_t, t)$ for finite-horizon return.
\begin{equation}
L^{\lambda}_V = \frac{1}{Nh}\sum^N_{i=1}\sum^{h-1}_{t=0}\left(G^{\lambda,i}_t - V(s^i_t, t)\right)^2
\label{eq:value}
\end{equation}
where $G^{\lambda,i}_t$ is the estimated value computed through the TD($\lambda$) formulation \citep{sutton2018reinforcement}. Refer to Appendix~\ref{app:value} for more details.

The pseudo-code of CGPO is provided in Algorithm \ref{alg:cgpo}.

\begin{algorithm}[ht]
  \caption{Constrained Gradient-based Policy Optimization (CGPO)}
  \label{alg:cgpo}
  \begin{algorithmic}
    \STATE \textbf{Input:} Initialize policy $\pi_{\bm{\theta}_0}$, critic $V_{\bm{\phi}_0}$ and $V^c_{\bm{\psi}_0}$, radius $\hat\delta_0$, and number of iterations $K$.
    \FOR {$k=1,2,\ldots,K$}
    \STATE Sample a set of trajectories $\mathcal D = \{\tau\} \sim \pi_{\bm{\theta}_k}$.
    \STATE Compute the $\bm{g}_k$, $\bm{q}_k$ using \eqref{eq:loss_r} and \eqref{eq:loss_c}.
    \IF {$c_k^2/\bm{q}^\top_k\bm{q}_k -\hat\delta_k\geq0$ and $c_k > 0$}
    \STATE Update the policy as $\bm{\theta}_{k+1} = \bm{\theta}_{k} - \frac{\sqrt{\hat\delta_k}}{\|\bm{q}_k\|}\bm{q}_k$.
    \ELSIF {$c_k^2/\bm{q}^\top_k\bm{q}_k-\hat\delta_k\geq0$ and $c_k<0$}
    \STATE Update the policy as $\bm{\theta}_{k+1} = \bm{\theta}_{k} + \frac{\sqrt{\hat\delta_k}}{\|\bm{g}_k\|}\bm{g}_k$.
    \ELSE
    \STATE Compute dual variables $\lambda^*_k$, $\nu^*_k$ using Algorithm~\ref{alg:calc_dual_var}.
    \STATE Update the policy as $\bm{\theta}_{k+1} = \bm{\theta}_{k} + \frac{\bm{g}_k-\nu^*_k\bm{q}_k}{\lambda^*_k}$.
    \ENDIF
    \STATE Update $V_{\bm{\phi}_k}$, $V^c_{\bm{\psi}_k}$ using \eqref{eq:value}, and $\hat{\delta}_{k+1}$ using \eqref{eq:update_radius}.
    \ENDFOR
    \STATE \textbf{Output:} Policy $\pi_{\bm{\theta}_K}$.
  \end{algorithmic}
\end{algorithm}

 \begin{figure*}[ht]
    \centering
    \includegraphics[width=0.98\linewidth]{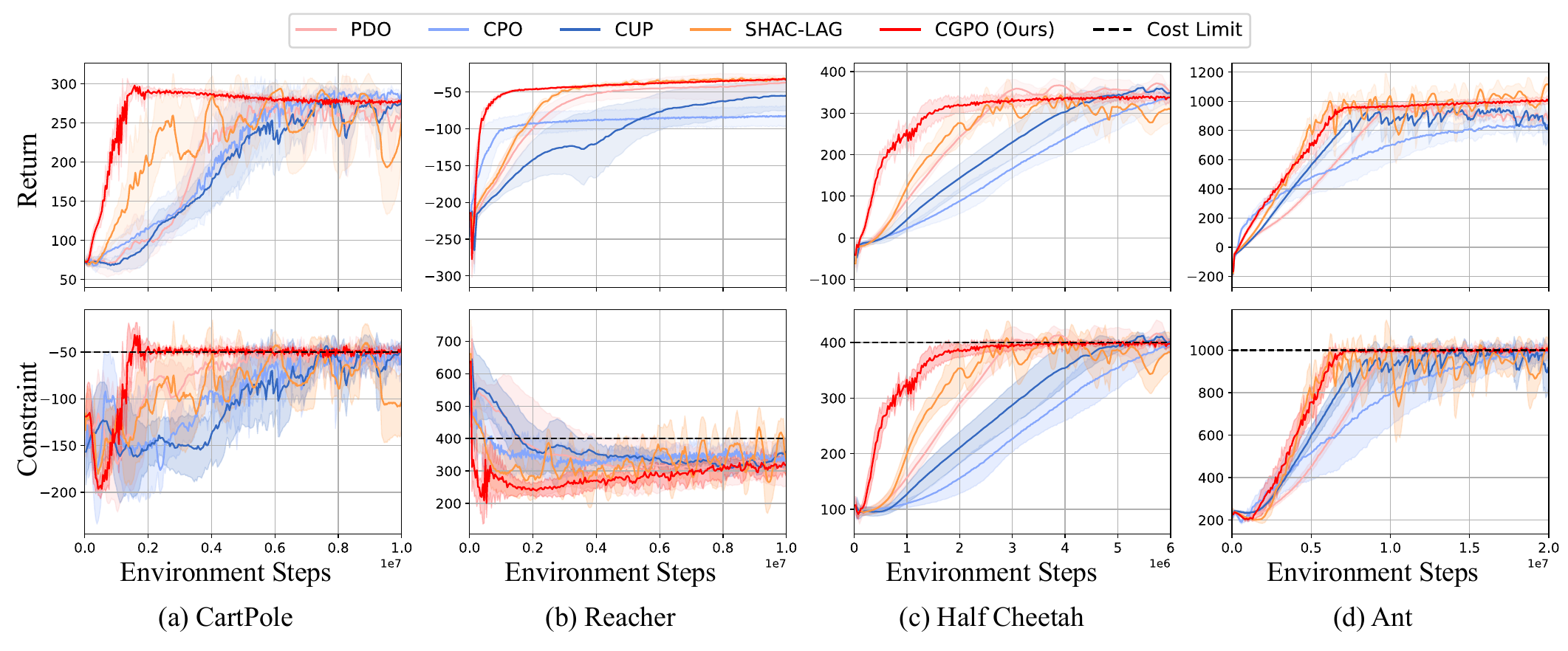}
    \vspace{-0.5em}
    \caption{Training curves of certain algorithms on different tasks, showing episodic return and constraint for 5 random seeds. Solid lines represent the mean, while the shaded areas indicate variance, without any smoothing to the curves. CGPO demonstrates superior efficiency in improvement and constraint satisfaction. The rest of the training curves can be found in Appendix~\ref{app:more_results}.}
    \label{fig:main}
\end{figure*}

\begin{table*}[ht]
\centering
\caption{The number of environmental steps to converge to a feasible optimal solution (Conv. Steps), and the proportion of constraint violations during updates in safety-critical areas (Vio. Ratio) for various algorithms. CGPO significantly outperforms others in both metrics.  
The calculation methods of metrics are explained in Appendix~\ref{app:eval_metrics}.}
\label{tab:analysis}
\resizebox{0.98\textwidth}{!}{
\begin{threeparttable}
\begin{tabular}{ccccccccc}
    \toprule
     & \multicolumn{2}{c}{CartPole} & \multicolumn{2}{c}{Reacher}& \multicolumn{2}{c}{HalfCheetah}  & \multicolumn{2}{c}{Ant}\\
     \cmidrule(lr){2-3}
     \cmidrule(lr){4-5}
     \cmidrule(lr){6-7}
     \cmidrule(lr){8-9}
     {Algorithms}  & Conv. Steps  $\downarrow$  & Vio. Ratio (\%)  $\downarrow$  & Conv. Steps  $\downarrow$  & Vio. Ratio (\%)  $\downarrow$ & Conv. Steps  $\downarrow$  & Vio. Ratio (\%)  $\downarrow$ & Conv. Steps  $\downarrow$  & Vio. Ratio (\%)  $\downarrow$   \\
     \midrule
PDO & 4.99e+06 & 12.01 & 3.99e+06 & 64.42 & 3.43e+06 & 74.58 & 1.70e+07 & 51.84 \\
APPO & 5.95e+06 & 6.59 & 4.22e+06 & 46.04 & 3.58e+06 & 18.33 & 1.08e+07 & 9.54 \\
CPO & 6.72e+06 & 5.96 & 1.98e+06 & 30.68 & 5.81e+06 & -{$^\dagger$} & 1.69e+07 & 24.56 \\
CUP & 6.84e+06 & 19.21 & 7.76e+06 & 56.12 & 5.12e+06 & 57.40 & 9.29e+06 & 29.79 \\
BPTT-LAG & 3.84e+06 & 23.99 & 2.73e+06 & 67.88 & 2.32e+06 & 61.56 & 2.13e+07 & -{$^\dagger$} \\
SHAC-LAG & 2.32e+06 & 46.31 & 2.96e+06 & 50.59 & 3.15e+06 & 65.90 & 1.67e+07 & 57.53 \\
\midrule
\textbf{CGPO}  & \textbf{1.77e+06} & \textbf{3.96} & \textbf{1.15e+06} & \textbf{14.29} & \textbf{2.02e+06} & \textbf{7.31} & \textbf{8.52e+06} &\textbf{6.44} \\
\bottomrule
\end{tabular}
\begin{tablenotes}
    \raggedright
    \item $\dagger$: The empty entries result from the corresponding algorithms failing to find the optimal feasible policy, leading to a lack of updates in the safety-critical region.\\
\end{tablenotes}
\end{threeparttable}
}
\end{table*}

\section{Experiments}
\label{sec:exp}

We conduct experiments to validate the effectiveness of our proposed CGPO. Our focus is primarily on four aspects:
\begin{itemize}[leftmargin=*]
    \item CGPO outperforms the baseline algorithms in Safe RL, demonstrating more efficient performance improvement and more precise constraint satisfaction (Section~\ref{sec:main_results}).
    \item CGPO employs the GBE method to obtain accurate estimations, unlike the ABE method fails (Section~\ref{sec:exp_error}).
    \item CGPO can overcome the differentiability requirements through Model-based approaches (Section~\ref{sec:exp_model}).
    \item CGPO can achieve more stable constraint convergence through an adaptive trust region radius (Section~\ref{sec:exp_adaptive}).
\end{itemize}

\subsection{Experimental Details}

Please refer to Appendix~\ref{app:impl_task} for the detailed implementation of experimental tasks and Appendix~\ref{app:exp_details} for baseline algorithms.

\textbf{Differentiable Safe RL environments.} Due to the lack of differentiable Safe RL environments, we develop a series of constrained differentiable tasks on an open-source differentiable physics engine Brax \citep{freeman2021brax}.
These tasks are based on four differentiable robotic control tasks in Brax (CartPole, Reacher, Half Cheetah, and Ant), with the addition of two common constraints: limiting position\citep{achiam2017constrained,ji2023safety} and limiting velocity\citep{zhang2020first}.
It is important to note that while adding these constraints, we maintained the differentiability of the physical dynamics.

\textbf{Baselines.} We compared CGPO with two categories of algorithms.
First, we compare it with traditional Safe RL methods, including both classic and latest Primal-Dual methods: PDO~\citep{chow2018risk} and APPO~\citep{dai2023augmented}, as well as classic and latest Primal methods: CPO~\citep{achiam2017constrained} and CUP~\citep{yang2020projection}.
Secondly, we combined the current SOTA differentiable algorithms, BPTT~\citep{mozer2013focused} and SHAC~\citep{Jie2022shac}, with the Lagrangian method to satisfy constraints.
Specifically, we used the Lagrange multiplier $\lambda$ to trade off the objective $\mathcal{J}_R$ and the constraint $\mathcal{J}_C$.
The refined algorithms are later termed as BPTT-Lag and SHAC-Lag.

\subsection{Overall Performance} \label{sec:main_results}

Figure \ref{fig:main} shows the learning curves of different algorithms across various tasks.
We observe that CGPO not only converges to a constraint-satisfying policy more quickly and stably than baseline algorithms but also demonstrates enhanced efficiency in improving performance.

\textbf{Compared to conventional Safe RL algorithms,} CGPO demonstrates superiority in both constraint satisfaction and sample efficiency.
Firstly, as illustrated in Figure~\ref{fig:main}, CGPO exhibits more precise control over constraints.
In scenarios where constraint satisfaction conflicts with performance improvement (i.e., CartPole), CGPO can strictly adhere to constraint thresholds to achieve higher performances.
Additionally, as shown in Table~\ref{tab:analysis}, CGPO significantly reduces the proportion of constraint violations in critical safety zones compared to other baseline algorithms.
The advantage of CGPO in constraint satisfaction largely sources from the accuracy of the GBE method, which will be further discussed in Section~\ref{sec:exp_error}.
Secondly, CGPO demonstrates higher sample efficiency by directly employing analytical gradients, as opposed to Monte Carlo sampling estimation.
As shown in Table ~\ref{tab:analysis}, CGPO requires less than 28.8\%-91.7\% of samples at convergence compared to baselines.
This finding aligns with previous works~\citep{mora2021pods, Jie2022shac}.

\textbf{Compared to SHAC-Lag and BPTT-Lag,} CGPO exhibits more stable convergence.
As Figure~\ref{fig:main} illustrates, while SHAC-Lag and BPTT-Lag achieve efficient sample utilization, they oscillate near the threshold.
This behavior sources from an inherent issue of Lagrangian methods: the delayed response of dual variables to constraint violations~\citep{platt1987constrained, wah2000improving}.
Such delays result in inherent oscillations and cost overshoots.

\begin{figure}[t]
    \centering
    \includegraphics[width=0.98\linewidth]{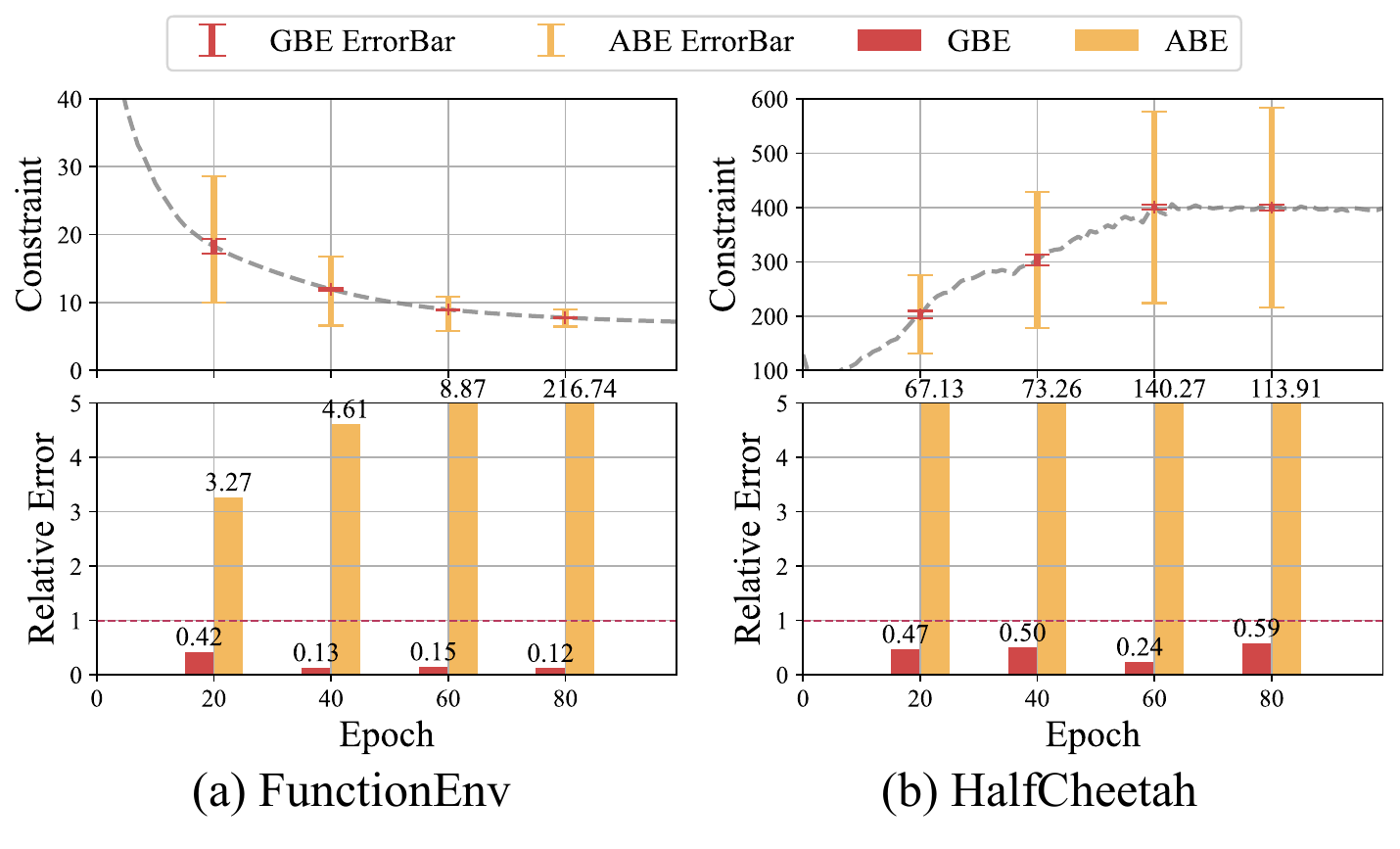}
    \caption{GBE and ABE errors across four training stages for equal step length updates and averaging over 100 repetitions, where $\text{Relative Error} = \frac{\text{Estimation Error}}{\text{Constraint Change}}$. GBE effectively predicts the constraint function of the next policy without failing like ABE.}
    \label{fig:errors}
\end{figure}

\begin{figure}[t]
    \centering
    \includegraphics[width=0.98\linewidth]{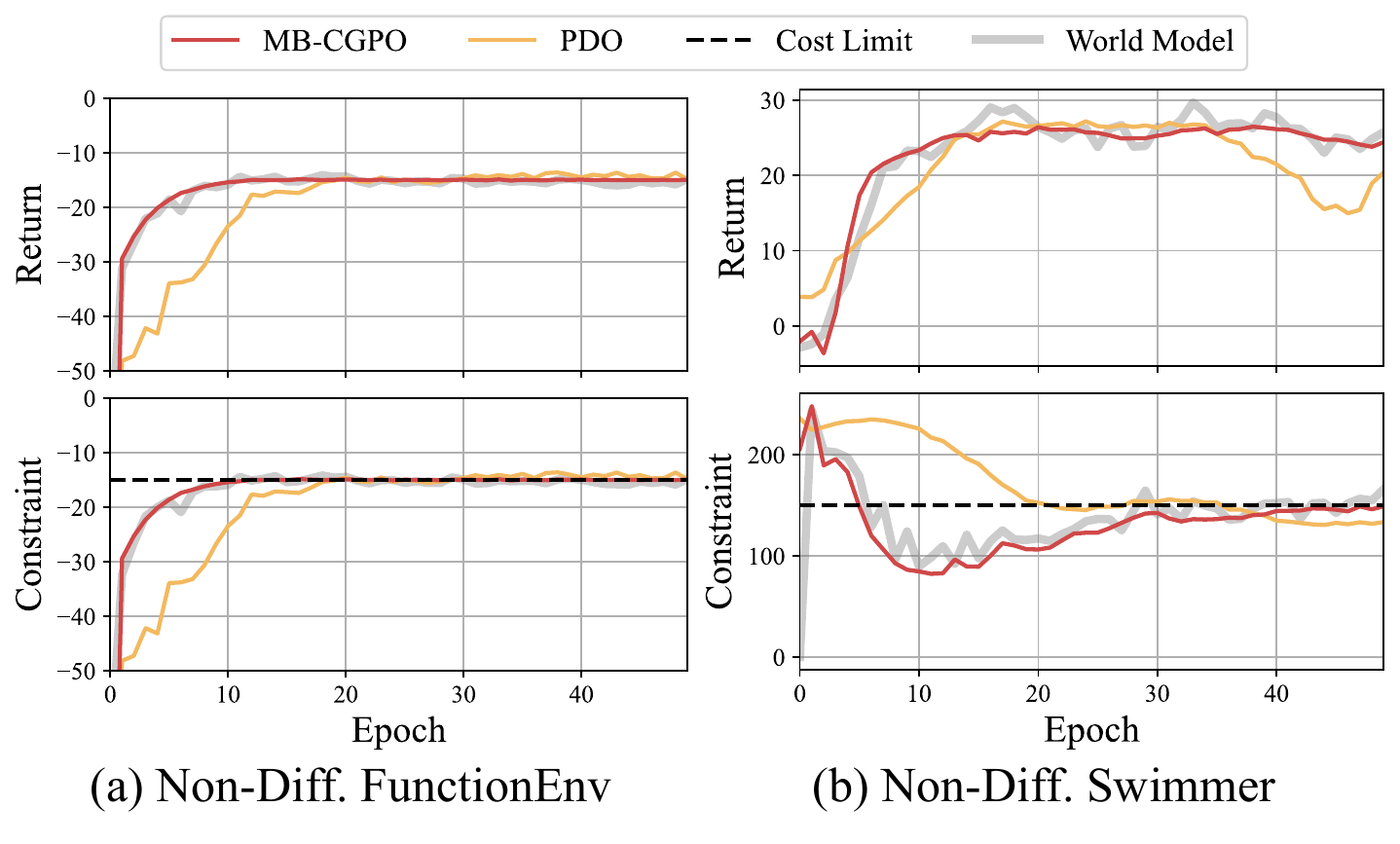}
    \caption{CGPO uses gradients from the World Models in non-differentiable tasks. The grey shading represents predictions from the World Model.}
    \label{fig:model_based}
\end{figure}

\subsection{Ablation on Estimation Errors}\label{sec:exp_error}

The enhanced performance of CGPO in our experiments primarily results from the GBE method's accurate estimation of finite-horizon constraints for future policies, a task at which the ABE method fails.
Our ablation results to compare GBE and ABE errors are illustrated in Fig.~\ref{fig:errors}.
We observe that the estimation error of GBE is much smaller than that of ABE; the relative error of ABE even exceeds 1.0, highlighting its almost ineffectiveness.
Thus, solutions from the ABE-based surrogate optimization problem are not as feasible for the primal Safe RL problem, preventing traditional algorithms from precisely meeting safety constraints.

In the HalfCheetah task, errors are higher than in the FunctionEnv task, primarily due to the reduced system differentiability.
This exposes a limitation of CGPO, its potential failure in poorly differentiable systems.
CGPO has the potential to overcome it by training a World Model for gradient provision, as discussed in the next section.

\subsection{Ablation on World Model Augmentation}\label{sec:exp_model}

In systems with limited differentiability, we can access analytic gradients by training a World Model.
The fundamental implementation of the model-based CGPO algorithm is presented by Algorithm~\ref{alg:mb-cgpo}.
We evaluate its efficacy on two non-differentiable tasks, with specific settings described in Appendix~\ref{app:model_based}.
As shown in Figure~\ref{fig:model_based}, the implementation of a World Model enables CGPO to achieve improved performance while adhering to safety constraints in environments with limited differentiability.
This demonstrates that the World Model has the potential to expand CGPO to a broader range of applications, overcoming the limitations of environmental differentiability.
Nonetheless, incorporating the World Model incurs additional time and computational overhead, necessitating further optimization.

\begin{figure}[t]
    \centering
    \includegraphics[width=0.98\linewidth]{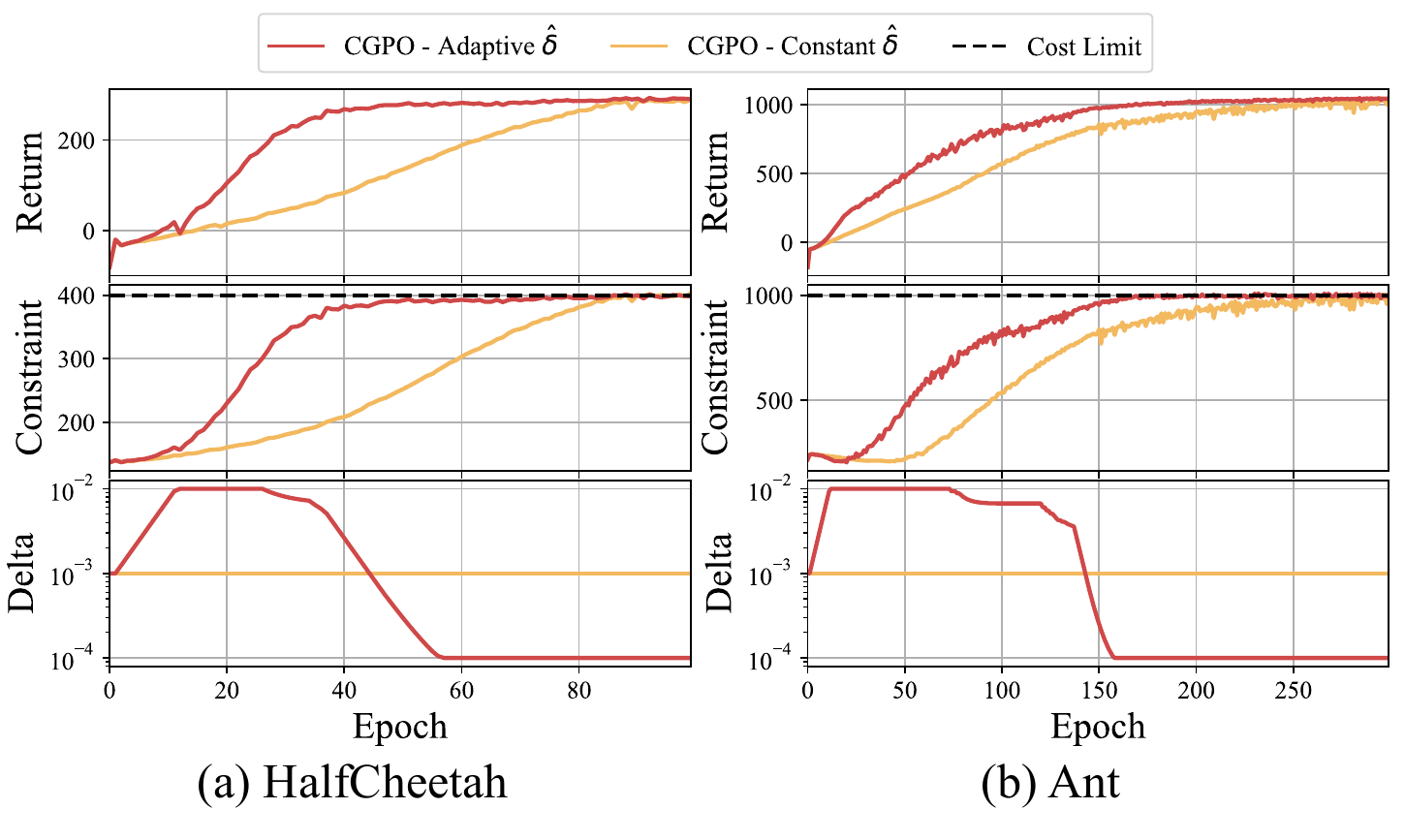}
    \vspace{-0.7em}
    \caption{Adaptive trust region radius $\hat{\delta}$ enables CGPO to achieve better performance improvements and constraint satisfaction.}
    \label{fig:adaptive}
    \vspace{-0.5em}
\end{figure}

\subsection{Ablation on Adaptive Trust Region Radius} \label{sec:exp_adaptive}

Our method dynamically adjusts the trust region radius $\hat{\delta}$ to maintain estimation error within a tolerable limit.
Adaptive $\hat{\delta}$, as Figure~\ref{fig:adaptive} demonstrates, enhances performance and ensures more accurate adherence to constraints than a static radius.
The learning curve of the adaptive radius has two phases: initially, it prioritizes improvements in efficient performance, ensuring consistency between actual and predicted performance changes, while achieving more larger updates.
Later, as policy updates approach constraint limits, the focus shifts to reducing estimation errors within the constraint budget, thereby preventing violations.
{More ablation studies on the hyper-parameters of radius adaptation can be found in Appendix~\ref{app:adaptation}}

\section{Limitation and Future Works} \label{app:limitation}

Although our algorithm demonstrates theoretical and empirical superiority over baseline algorithms, it has certain limitations. Like other differentiable reinforcement learning methods, CGPO relies on environmental differentiability. As discussed in Section~\ref{sec:exp_error}, systems with poor differentiability are susceptible to environmental noise, resulting in increased system errors and reduced precision in GBE estimation. While CGPO mitigates the upper bound of worst-case estimation errors through an adaptive trust region radius method, an excessively small radius can reduce update efficiency.

We identify three promising solutions to address this limitation.
First, advancing differentiable physics engines to enhance environmental differentiability~\citep{degrave2019differentiable,freeman2021brax, howell2022dojo}.
Second, developing algorithms to minimize system errors at the software level~\citep{metz2021gradients}, such as the reparameterization technique~\citep{lee2018reparameterization}.
Third, providing analytic gradients for systems with poor differentiability through more accurate world model training~\citep{parmas2023model}.
In Section~\ref{sec:exp_model}, we discuss this approach and its basic implementation using a two-layer Multi-Layer Perceptron (MLP)~\citep{popescu2009multilayer}, resulting in low-precision gradients from the world model.

For future work, it is crucial to develop a more universal framework based on our algorithm that incorporates the three aforementioned approaches to mitigate existing limitations.
Additionally, we plan to extend the application of our algorithm to fields beyond robotic control~\citep{he2024personalized,wang2024development}, such as the safety alignment~\citep{ji2023ai} of large language models~\citep{ji2024beavertails,dai2024safe,ji2024aligner} and autonomous driving~\citep{muhammad2020deep,gao_review_2023}.

\section{Conclusion}
Since current deep Safe RL algorithms uniformly apply an infinite-horizon approach to all constraints, they fail to guarantee the feasibility of each update.
In this paper, we introduce the \gbe{} (GBE) method to bridge the gap of finite-horizon constraint estimation in deep Safe RL.
Based on GBE, we formulate a surrogate optimization problem and propose the first deep Safe RL algorithm, named \ours{} (CGPO), able to handle tasks with finite-horizon constraints.
CGPO leverages precise GBE estimations to ensure the efficiency and feasibility of each update.
We also present a differentiable Safe RL environment for algorithm testing.
Comparative evaluations reveal that our algorithm outperforms baselines in terms of update efficiency and constraint satisfaction.

\section*{Impact Statement}
This paper presents work whose goal is to advance the field of machine learning, specifically focusing on RL within differentiable simulation environments or world models.
There are many potential societal consequences of our work, none of which we feel must be specifically highlighted here.

\bibliography{example_paper}

\begin{thebibliography}{63}
\providecommand{\natexlab}[1]{#1}
\providecommand{\url}[1]{\texttt{#1}}
\expandafter\ifx\csname urlstyle\endcsname\relax
  \providecommand{\doi}[1]{doi: #1}\else
  \providecommand{\doi}{doi: \begingroup \urlstyle{rm}\Url}\fi

\bibitem[Achiam et~al.(2017)Achiam, Held, Tamar, and Abbeel]{achiam2017constrained}
Achiam, J., Held, D., Tamar, A., and Abbeel, P.
\newblock Constrained policy optimization, 2017.

\bibitem[Altman(1999)]{altman1999constrained}
Altman, E.
\newblock \emph{Constrained Markov decision processes}, volume~7.
\newblock CRC press, 1999.

\bibitem[As et~al.(2022)As, Usmanova, Curi, and Krause]{as2022constrained}
As, Y., Usmanova, I., Curi, S., and Krause, A.
\newblock Constrained policy optimization via bayesian world models.
\newblock \emph{arXiv preprint arXiv:2201.09802}, 2022.

\bibitem[Bogue(2017)]{bogue2017robots}
Bogue, R.
\newblock Robots that interact with humans: a review of safety technologies and standards.
\newblock \emph{Industrial Robot: An International Journal}, 44\penalty0 (4):\penalty0 395--400, 2017.

\bibitem[Chow et~al.(2018)Chow, Ghavamzadeh, Janson, and Pavone]{chow2018risk}
Chow, Y., Ghavamzadeh, M., Janson, L., and Pavone, M.
\newblock Risk-constrained reinforcement learning with percentile risk criteria.
\newblock \emph{Journal of Machine Learning Research}, 18\penalty0 (167):\penalty0 1--51, 2018.

\bibitem[Clavera et~al.(2020)Clavera, Fu, and Abbeel]{clavera2020model}
Clavera, I., Fu, V., and Abbeel, P.
\newblock Model-augmented actor-critic: Backpropagating through paths.
\newblock \emph{arXiv preprint arXiv:2005.08068}, 2020.

\bibitem[Dai et~al.(2023)Dai, Ji, Yang, Zheng, and Pan]{dai2023augmented}
Dai, J., Ji, J., Yang, L., Zheng, Q., and Pan, G.
\newblock Augmented proximal policy optimization for safe reinforcement learning.
\newblock In \emph{Proceedings of the AAAI Conference on Artificial Intelligence}, volume~37, pp.\  7288--7295, 2023.

\bibitem[Dai et~al.(2024)Dai, Pan, Sun, Ji, Xu, Liu, Wang, and Yang]{dai2024safe}
Dai, J., Pan, X., Sun, R., Ji, J., Xu, X., Liu, M., Wang, Y., and Yang, Y.
\newblock Safe {RLHF}: Safe reinforcement learning from human feedback.
\newblock In \emph{The Twelfth International Conference on Learning Representations}, 2024.
\newblock URL \url{https://openreview.net/forum?id=TyFrPOKYXw}.

\bibitem[Degrave et~al.(2019)Degrave, Hermans, Dambre, et~al.]{degrave2019differentiable}
Degrave, J., Hermans, M., Dambre, J., et~al.
\newblock A differentiable physics engine for deep learning in robotics.
\newblock \emph{Frontiers in neurorobotics}, pp.\ ~6, 2019.

\bibitem[ElDahshan et~al.(2022)ElDahshan, Farouk, and Mofreh]{eldahshan2022deep}
ElDahshan, K.~A., Farouk, H., and Mofreh, E.
\newblock Deep reinforcement learning based video games: A review.
\newblock In \emph{2022 2nd International Mobile, Intelligent, and Ubiquitous Computing Conference (MIUCC)}, pp.\  302--309. IEEE, 2022.

\bibitem[Freeman et~al.(2021)Freeman, Frey, Raichuk, Girgin, Mordatch, and Bachem]{freeman2021brax}
Freeman, C.~D., Frey, E., Raichuk, A., Girgin, S., Mordatch, I., and Bachem, O.
\newblock Brax - a differentiable physics engine for large scale rigid body simulation.
\newblock In \emph{Thirty-fifth Conference on Neural Information Processing Systems Datasets and Benchmarks Track (Round 1)}, 2021.

\bibitem[Gao et~al.(2023)Gao, Wan, Chen, Wang, Li, Jiang, Mei, Luo, and Li]{gao_review_2023}
Gao, B., Wan, K., Chen, Q., Wang, Z., Li, R., Jiang, Y., Mei, R., Luo, Y., and Li, K.
\newblock A {Review} and {Outlook} on {Predictive} {Cruise} {Control} of {Vehicles} and {Typical} {Applications} {Under} {Cloud} {Control} {System}.
\newblock \emph{Machine Intelligence Research}, 20\penalty0 (5):\penalty0 614--639, October 2023.
\newblock ISSN 2731-5398.
\newblock \doi{10.1007/s11633-022-1395-3}.
\newblock URL \url{https://doi.org/10.1007/s11633-022-1395-3}.

\bibitem[Garc{\i}a \& Fern{\'a}ndez(2015)Garc{\i}a and Fern{\'a}ndez]{garcia2015comprehensive}
Garc{\i}a, J. and Fern{\'a}ndez, F.
\newblock A comprehensive survey on safe reinforcement learning.
\newblock \emph{Journal of Machine Learning Research}, 16\penalty0 (1):\penalty0 1437--1480, 2015.

\bibitem[Gronauer(2022)]{gronauer2022bullet}
Gronauer, S.
\newblock Bullet-safety-gym: A framework for constrained reinforcement learning.
\newblock 2022.

\bibitem[Guin \& Bhatnagar(2023)Guin and Bhatnagar]{guin2023policy}
Guin, S. and Bhatnagar, S.
\newblock A policy gradient approach for finite horizon constrained markov decision processes.
\newblock In \emph{2023 62nd IEEE Conference on Decision and Control (CDC)}, pp.\  3353--3359. IEEE, 2023.

\bibitem[He et~al.(2024)He, Hu, Yang, and Lv]{he2024personalized}
He, X., Hu, Z., Yang, H., and Lv, C.
\newblock Personalized robotic control via constrained multi-objective reinforcement learning.
\newblock \emph{Neurocomputing}, 565:\penalty0 126986, 2024.

\bibitem[Howell et~al.(2022)Howell, Cleac'h, Br{\"u}digam, Kolter, Schwager, and Manchester]{howell2022dojo}
Howell, T.~A., Cleac'h, S.~L., Br{\"u}digam, J., Kolter, J.~Z., Schwager, M., and Manchester, Z.
\newblock Dojo: A differentiable physics engine for robotics.
\newblock \emph{arXiv preprint arXiv:2203.00806}, 2022.

\bibitem[Jaisson(2022)]{jaisson2022deep}
Jaisson, T.
\newblock Deep differentiable reinforcement learning and optimal trading.
\newblock \emph{Quantitative Finance}, 22\penalty0 (8):\penalty0 1429--1443, 2022.

\bibitem[Ji et~al.(2023{\natexlab{a}})Ji, Qiu, Chen, Zhang, Lou, Wang, Duan, He, Zhou, Zhang, et~al.]{ji2023ai}
Ji, J., Qiu, T., Chen, B., Zhang, B., Lou, H., Wang, K., Duan, Y., He, Z., Zhou, J., Zhang, Z., et~al.
\newblock Ai alignment: A comprehensive survey.
\newblock \emph{arXiv preprint arXiv:2310.19852}, 2023{\natexlab{a}}.

\bibitem[Ji et~al.(2023{\natexlab{b}})Ji, Zhang, Zhou, Pan, Huang, Sun, Geng, Zhong, Dai, and Yang]{ji2023safety}
Ji, J., Zhang, B., Zhou, J., Pan, X., Huang, W., Sun, R., Geng, Y., Zhong, Y., Dai, J., and Yang, Y.
\newblock Safety-gymnasium: A unified safe reinforcement learning benchmark.
\newblock \emph{arXiv preprint arXiv:2310.12567}, 2023{\natexlab{b}}.

\bibitem[Ji et~al.(2023{\natexlab{c}})Ji, Zhou, Zhang, Dai, Pan, Sun, Huang, Geng, Liu, and Yang]{ji2023omnisafe}
Ji, J., Zhou, J., Zhang, B., Dai, J., Pan, X., Sun, R., Huang, W., Geng, Y., Liu, M., and Yang, Y.
\newblock Omnisafe: An infrastructure for accelerating safe reinforcement learning research.
\newblock \emph{arXiv preprint arXiv:2305.09304}, 2023{\natexlab{c}}.

\bibitem[Ji et~al.(2024{\natexlab{a}})Ji, Chen, Lou, Hong, Zhang, Pan, Dai, and Yang]{ji2024aligner}
Ji, J., Chen, B., Lou, H., Hong, D., Zhang, B., Pan, X., Dai, J., and Yang, Y.
\newblock Aligner: Achieving efficient alignment through weak-to-strong correction.
\newblock \emph{arXiv preprint arXiv:2402.02416}, 2024{\natexlab{a}}.

\bibitem[Ji et~al.(2024{\natexlab{b}})Ji, Liu, Dai, Pan, Zhang, Bian, Chen, Sun, Wang, and Yang]{ji2024beavertails}
Ji, J., Liu, M., Dai, J., Pan, X., Zhang, C., Bian, C., Chen, B., Sun, R., Wang, Y., and Yang, Y.
\newblock Beavertails: Towards improved safety alignment of llm via a human-preference dataset.
\newblock \emph{Advances in Neural Information Processing Systems}, 36, 2024{\natexlab{b}}.

\bibitem[Jie~Xu et~al.(2022)Jie~Xu, Makoviychuk, and Narang]{Jie2022shac}
Jie~Xu, V., Makoviychuk, Y., and Narang, F.~R.
\newblock Accelerated policy learning with parallel differentiable simulation.
\newblock In \emph{ICLR}, 2022.

\bibitem[Kakade \& Langford(2002)Kakade and Langford]{kakade2002approximately}
Kakade, S. and Langford, J.
\newblock Approximately optimal approximate reinforcement learning.
\newblock In \emph{Proceedings of the Nineteenth International Conference on Machine Learning}, ICML '02, pp.\  267–274, San Francisco, CA, USA, 2002. Morgan Kaufmann Publishers Inc.
\newblock ISBN 1558608737.

\bibitem[Kalagarla et~al.(2021)Kalagarla, Jain, and Nuzzo]{kalagarla2021sample}
Kalagarla, K.~C., Jain, R., and Nuzzo, P.
\newblock A sample-efficient algorithm for episodic finite-horizon mdp with constraints.
\newblock In \emph{Proceedings of the AAAI Conference on Artificial Intelligence}, volume~35, pp.\  8030--8037, 2021.

\bibitem[Lee et~al.(2018)Lee, Yu, and Yang]{lee2018reparameterization}
Lee, W., Yu, H., and Yang, H.
\newblock Reparameterization gradient for non-differentiable models.
\newblock \emph{Advances in Neural Information Processing Systems}, 31, 2018.

\bibitem[Liu et~al.(2022)Liu, Cen, Isenbaev, Liu, Wu, Li, and Zhao]{liu2022constrained}
Liu, Z., Cen, Z., Isenbaev, V., Liu, W., Wu, S., Li, B., and Zhao, D.
\newblock Constrained variational policy optimization for safe reinforcement learning.
\newblock In \emph{International Conference on Machine Learning}, pp.\  13644--13668. PMLR, 2022.

\bibitem[Meng et~al.(2022)Meng, Zheng, Shi, and Pan]{9334437}
Meng, W., Zheng, Q., Shi, Y., and Pan, G.
\newblock An off-policy trust region policy optimization method with monotonic improvement guarantee for deep reinforcement learning.
\newblock \emph{IEEE Transactions on Neural Networks and Learning Systems}, 33\penalty0 (5):\penalty0 2223--2235, 2022.
\newblock \doi{10.1109/TNNLS.2020.3044196}.

\bibitem[Metz et~al.(2021)Metz, Freeman, Schoenholz, and Kachman]{metz2021gradients}
Metz, L., Freeman, C.~D., Schoenholz, S.~S., and Kachman, T.
\newblock Gradients are not all you need.
\newblock \emph{arXiv preprint arXiv:2111.05803}, 2021.

\bibitem[Mohamed et~al.(2020)Mohamed, Rosca, Figurnov, and Mnih]{mohamed2020monte}
Mohamed, S., Rosca, M., Figurnov, M., and Mnih, A.
\newblock Monte carlo gradient estimation in machine learning.
\newblock \emph{The Journal of Machine Learning Research}, 21\penalty0 (1):\penalty0 5183--5244, 2020.

\bibitem[Mora et~al.(2021)Mora, Peychev, Ha, Vechev, and Coros]{mora2021pods}
Mora, M. A.~Z., Peychev, M., Ha, S., Vechev, M., and Coros, S.
\newblock Pods: Policy optimization via differentiable simulation.
\newblock In Meila, M. and Zhang, T. (eds.), \emph{Proceedings of the 38th International Conference on Machine Learning}, volume 139 of \emph{Proceedings of Machine Learning Research}, pp.\  7805--7817. PMLR, 18--24 Jul 2021.

\bibitem[Mozer(1995)]{bptt}
Mozer, M.
\newblock A focused backpropagation algorithm for temporal pattern recognition.
\newblock \emph{Complex Systems}, 3, 01 1995.

\bibitem[Mozer(2013)]{mozer2013focused}
Mozer, M.~C.
\newblock A focused backpropagation algorithm for temporal pattern recognition.
\newblock In \emph{Backpropagation}, pp.\  137--169. Psychology Press, 2013.

\bibitem[Muhammad et~al.(2020)Muhammad, Ullah, Lloret, Del~Ser, and de~Albuquerque]{muhammad2020deep}
Muhammad, K., Ullah, A., Lloret, J., Del~Ser, J., and de~Albuquerque, V. H.~C.
\newblock Deep learning for safe autonomous driving: Current challenges and future directions.
\newblock \emph{IEEE Transactions on Intelligent Transportation Systems}, 22\penalty0 (7):\penalty0 4316--4336, 2020.

\bibitem[Nilsson et~al.(2017)Nilsson, Fredriksson, and Coelingh]{NILSSON20179083}
Nilsson, J., Fredriksson, J., and Coelingh, E.
\newblock Trajectory planning with miscellaneous safety critical zones**this work was supported by ffi - strategic vehicle research and innovation.
\newblock \emph{IFAC-PapersOnLine}, 50\penalty0 (1):\penalty0 9083--9088, 2017.
\newblock ISSN 2405-8963.
\newblock \doi{https://doi.org/10.1016/j.ifacol.2017.08.1649}.
\newblock URL \url{https://www.sciencedirect.com/science/article/pii/S2405896317322541}.
\newblock 20th IFAC World Congress.

\bibitem[Okamura et~al.(2000)Okamura, Smaby, and Cutkosky]{okamura2000overview}
Okamura, A.~M., Smaby, N., and Cutkosky, M.~R.
\newblock An overview of dexterous manipulation.
\newblock In \emph{Proceedings 2000 ICRA. Millennium Conference. IEEE International Conference on Robotics and Automation. Symposia Proceedings (Cat. No. 00CH37065)}, volume~1, pp.\  255--262. IEEE, 2000.

\bibitem[Ouyang et~al.(2022)Ouyang, Wu, Jiang, Almeida, Wainwright, Mishkin, Zhang, Agarwal, Slama, Ray, et~al.]{ouyang2022training}
Ouyang, L., Wu, J., Jiang, X., Almeida, D., Wainwright, C., Mishkin, P., Zhang, C., Agarwal, S., Slama, K., Ray, A., et~al.
\newblock Training language models to follow instructions with human feedback.
\newblock \emph{Advances in Neural Information Processing Systems}, 35:\penalty0 27730--27744, 2022.

\bibitem[Parmas et~al.(2023{\natexlab{a}})Parmas, Seno, and Aoki]{parmas2023model}
Parmas, P., Seno, T., and Aoki, Y.
\newblock Model-based reinforcement learning with scalable composite policy gradient estimators.
\newblock In \emph{International Conference on Machine Learning}, pp.\  27346--27377. PMLR, 2023{\natexlab{a}}.

\bibitem[Parmas et~al.(2023{\natexlab{b}})Parmas, Seno, and Aoki]{pmlr-v202-parmas23a}
Parmas, P., Seno, T., and Aoki, Y.
\newblock Model-based reinforcement learning with scalable composite policy gradient estimators.
\newblock In Krause, A., Brunskill, E., Cho, K., Engelhardt, B., Sabato, S., and Scarlett, J. (eds.), \emph{Proceedings of the 40th International Conference on Machine Learning}, volume 202 of \emph{Proceedings of Machine Learning Research}, pp.\  27346--27377. PMLR, 23--29 Jul 2023{\natexlab{b}}.
\newblock URL \url{https://proceedings.mlr.press/v202/parmas23a.html}.

\bibitem[Platt \& Barr(1987)Platt and Barr]{platt1987constrained}
Platt, J. and Barr, A.
\newblock Constrained differential optimization.
\newblock In \emph{Neural Information Processing Systems}, 1987.

\bibitem[Popescu et~al.(2009)Popescu, Balas, Perescu-Popescu, and Mastorakis]{popescu2009multilayer}
Popescu, M.-C., Balas, V.~E., Perescu-Popescu, L., and Mastorakis, N.
\newblock Multilayer perceptron and neural networks.
\newblock \emph{WSEAS Transactions on Circuits and Systems}, 8\penalty0 (7):\penalty0 579--588, 2009.

\bibitem[Puterman(1990)]{puterman1990markov}
Puterman, M.~L.
\newblock Markov decision processes.
\newblock \emph{Handbooks in operations research and management science}, 2:\penalty0 331--434, 1990.

\bibitem[Puterman(2014)]{puterman2014markov}
Puterman, M.~L.
\newblock \emph{Markov decision processes: discrete stochastic dynamic programming}.
\newblock John Wiley \& Sons, 2014.

\bibitem[Rafailov et~al.(2023)Rafailov, Sharma, Mitchell, Ermon, Manning, and Finn]{rafailov2023direct}
Rafailov, R., Sharma, A., Mitchell, E., Ermon, S., Manning, C.~D., and Finn, C.
\newblock Direct preference optimization: Your language model is secretly a reward model.
\newblock \emph{arXiv preprint arXiv:2305.18290}, 2023.

\bibitem[Ray et~al.(2019)Ray, Achiam, and Amodei]{ray2019benchmarking}
Ray, A., Achiam, J., and Amodei, D.
\newblock Benchmarking safe exploration in deep reinforcement learning.
\newblock \emph{arXiv preprint arXiv:1910.01708}, 7\penalty0 (1):\penalty0 2, 2019.

\bibitem[Schulman et~al.(2015)Schulman, Levine, Abbeel, Jordan, and Moritz]{schulman2015trust}
Schulman, J., Levine, S., Abbeel, P., Jordan, M., and Moritz, P.
\newblock Trust region policy optimization.
\newblock In \emph{International conference on machine learning}, pp.\  1889--1897. PMLR, 2015.

\bibitem[Shi et~al.(2019)Shi, Li, Cao, Yang, and Pan]{shi2019tbq}
Shi, L., Li, S., Cao, L., Yang, L., and Pan, G.
\newblock Tbq ($\sigma$): Improving efficiency of trace utilization for off-policy reinforcement learning.
\newblock \emph{arXiv preprint arXiv:1905.07237}, 2019.

\bibitem[Silver et~al.(2016)Silver, Huang, Maddison, Guez, Sifre, Van Den~Driessche, Schrittwieser, Antonoglou, Panneershelvam, Lanctot, et~al.]{silver2016mastering}
Silver, D., Huang, A., Maddison, C.~J., Guez, A., Sifre, L., Van Den~Driessche, G., Schrittwieser, J., Antonoglou, I., Panneershelvam, V., Lanctot, M., et~al.
\newblock Mastering the game of go with deep neural networks and tree search.
\newblock \emph{nature}, 529\penalty0 (7587):\penalty0 484--489, 2016.

\bibitem[Silver et~al.(2017)Silver, Schrittwieser, Simonyan, Antonoglou, Huang, Guez, Hubert, Baker, Lai, Bolton, et~al.]{silver2017mastering}
Silver, D., Schrittwieser, J., Simonyan, K., Antonoglou, I., Huang, A., Guez, A., Hubert, T., Baker, L., Lai, M., Bolton, A., et~al.
\newblock Mastering the game of go without human knowledge.
\newblock \emph{nature}, 550\penalty0 (7676):\penalty0 354--359, 2017.

\bibitem[Singh et~al.(2022)Singh, Kumar, and Singh]{singh2022reinforcement}
Singh, B., Kumar, R., and Singh, V.~P.
\newblock Reinforcement learning in robotic applications: a comprehensive survey.
\newblock \emph{Artificial Intelligence Review}, pp.\  1--46, 2022.

\bibitem[Stooke et~al.(2020)Stooke, Achiam, and Abbeel]{stooke2020responsive}
Stooke, A., Achiam, J., and Abbeel, P.
\newblock Responsive safety in reinforcement learning by pid lagrangian methods.
\newblock In \emph{International Conference on Machine Learning}, pp.\  9133--9143. PMLR, 2020.

\bibitem[Suh et~al.(2022)Suh, Simchowitz, Zhang, and Tedrake]{suh2022differentiable}
Suh, H.~J., Simchowitz, M., Zhang, K., and Tedrake, R.
\newblock Do differentiable simulators give better policy gradients?
\newblock In \emph{International Conference on Machine Learning}, pp.\  20668--20696. PMLR, 2022.

\bibitem[Sutton \& Barto(2018)Sutton and Barto]{sutton2018reinforcement}
Sutton, R.~S. and Barto, A.~G.
\newblock \emph{Reinforcement learning: An introduction}.
\newblock MIT press, 2018.

\bibitem[Wah et~al.(2000)Wah, Wang, Shang, and Wu]{wah2000improving}
Wah, B.~W., Wang, T., Shang, Y., and Wu, Z.
\newblock Improving the performance of weighted lagrange-multiplier methods for nonlinear constrained optimization.
\newblock \emph{Information Sciences}, 124\penalty0 (1-4):\penalty0 241--272, 2000.

\bibitem[Wang et~al.(2024)Wang, Zhang, Li, Cui, and Chen]{wang2024development}
Wang, Y., Zhang, M., Li, M., Cui, H., and Chen, X.
\newblock Development of a humanoid robot control system based on ar-bci and slam navigation.
\newblock \emph{Cognitive Neurodynamics}, pp.\  1--14, 2024.

\bibitem[Werling et~al.(2021)Werling, Omens, Lee, Exarchos, and Liu]{werling2021fast}
Werling, K., Omens, D., Lee, J., Exarchos, I., and Liu, C.~K.
\newblock Fast and feature-complete differentiable physics engine for articulated rigid bodies with contact constraints.
\newblock In \emph{Robotics: Science and Systems}, 2021.

\bibitem[Xian et~al.(2023)Xian, Zhu, Xu, Tung, Torralba, Fragkiadaki, and Gan]{xian2023fluidlab}
Xian, Z., Zhu, B., Xu, Z., Tung, H.-Y., Torralba, A., Fragkiadaki, K., and Gan, C.
\newblock Fluidlab: A differentiable environment for benchmarking complex fluid manipulation.
\newblock \emph{arXiv preprint arXiv:2303.02346}, 2023.

\bibitem[Xu et~al.(2022)Xu, Liu, Huang, Ding, Cen, Li, and Zhao]{xu2022trustworthy}
Xu, M., Liu, Z., Huang, P., Ding, W., Cen, Z., Li, B., and Zhao, D.
\newblock Trustworthy reinforcement learning against intrinsic vulnerabilities: Robustness, safety, and generalizability.
\newblock \emph{arXiv preprint arXiv:2209.08025}, 2022.

\bibitem[Yang et~al.(2018)Yang, Shi, Zheng, Meng, and Pan]{yang2018unified}
Yang, L., Shi, M., Zheng, Q., Meng, W., and Pan, G.
\newblock A unified approach for multi-step temporal-difference learning with eligibility traces in reinforcement learning.
\newblock \emph{arXiv preprint arXiv:1802.03171}, 2018.

\bibitem[Yang et~al.(2022)Yang, Ji, Dai, Zhang, Zhou, Li, Yang, and Pan]{yang2022constrained}
Yang, L., Ji, J., Dai, J., Zhang, L., Zhou, B., Li, P., Yang, Y., and Pan, G.
\newblock Constrained update projection approach to safe policy optimization.
\newblock \emph{Advances in Neural Information Processing Systems}, 35:\penalty0 9111--9124, 2022.

\bibitem[Yang et~al.(2020)Yang, Rosca, Narasimhan, and Ramadge]{yang2020projection}
Yang, T.-Y., Rosca, J., Narasimhan, K., and Ramadge, P.~J.
\newblock Projection-based constrained policy optimization.
\newblock \emph{arXiv preprint arXiv:2010.03152}, 2020.

\bibitem[Zhang et~al.(2020)Zhang, Vuong, and Ross]{zhang2020first}
Zhang, Y., Vuong, Q., and Ross, K.
\newblock First order constrained optimization in policy space.
\newblock \emph{Advances in Neural Information Processing Systems}, 33:\penalty0 15338--15349, 2020.

\end{thebibliography}
\bibliographystyle{icml2024}

\onecolumn
\part{Appendix}
\parttoc

\appendix

\newpage

\section{Supplementing Theoretical Aspects}
\subsection{The First-Order Approximation}
\label{app:first_order_error}
We employ the following first-order estimation to approximate the objective and constraint functions:
\begin{equation}
    \hat{\mathcal J}_f(\bm{\theta}_0+\bm{\delta}) = \mathcal J_f (\bm{\theta}_0) + \bm{\delta}^\top \nabla_{\bm{\theta}} \mathcal J_f(\bm{\theta}_0),
\end{equation}
where $\hat{\mathcal J}_f$ represents the both objective function $\hat{\mathcal J}_R$ and constraint function $\hat{\mathcal J}_C$:
\begin{equation}
    \hat{\mathcal J}_R(\bm{\theta}_0+\bm{\delta}) = \mathcal J_R (\bm{\theta}_0) + \bm{\delta}^\top \nabla_{\bm{\theta}} \mathcal J_R(\bm{\theta}_0),
    \label{eq:app_JR}
\end{equation}
\begin{equation}
    \hat{\mathcal J}_C(\bm{\theta}_0+\bm{\delta}) = \mathcal J_C (\bm{\theta}_0) + \bm{\delta}^\top \nabla_{\bm{\theta}} \mathcal J_C(\bm{\theta}_0).
    \label{eq:app_JC}
\end{equation}
This approximation introduces the following approximation error:
\TheEstError*
\begin{proof}
The function $\mathcal{J}_f(\bm{\theta})$ is expanded using a Taylor series at the point $\bm{\theta}_0+\bm{\delta}$, resulting in:
\begin{equation}
        \exists t_0 \in (0, 1),\quad
        \mathcal J_f(\bm{\theta}_0+\bm{\delta}) = \mathcal J_f (\bm{\theta}_0) + \bm{\delta}^\top \nabla_{\bm{\theta}} \mathcal J_f(\bm{\theta}_0) + \frac{1}{2} \bm{\delta}^\top\nabla_{\bm{\theta}}^2 \mathcal J_f(\bm{\theta}_0+t_0\bm{\delta})\bm{\delta}
 \end{equation}
where $\frac{1}{2} \bm{\delta}^\top\nabla_{\bm{\theta}}^2 \mathcal J_f(\bm{\theta}_0+t_0\bm{\delta})\bm{\delta}$ represents the Peano remainder term. Substituting this expanded result yields:
\begin{equation}
    \left|\hat{\mathcal J}_f(\bm{\theta}_0+\bm{\delta}) - {\mathcal J}_f(\bm{\theta}_0+\bm{\delta})\right|
    = \frac12 \big|\bm{\delta}^\top\nabla_{\bm{\theta}}^2 \mathcal J_f(\bm{\theta}_0+t_0\bm{\delta})\bm{\delta}\big|
    \leq\frac{1}{2} \epsilon \|\bm{\delta}\|^2_2.
\end{equation}
\end{proof}

\subsection{The Solution to Constrained Surrogate Sub-Problem}\label{app:solution}

Based on Theorem \ref{the:estimation_error}, we find that the estimation error is positively correlated with the update step size.
Therefore, by controlling the update step size, we can keep the estimation error within an acceptable range.
Given the trust region radius, we consider updated parameter $\bm{\theta_{k+1}}$ of the $k^{th}$ iteration within the trust region $\Theta_k=\left\{ \bm{\theta} \in \Theta \mid \|\bm{\theta} - \bm{\theta}_k\|^2 \leq \hat\delta \right\}$ to be credible.
By employing the approximations of the objective function in Equation~\eqref{eq:app_JR} and the constraint function in Equation~\eqref{eq:app_JC}, We transform the solution of the primal Safe RL problem into an iterative process of solving a series of sub-problems within predefined trust regions $\Theta_k$:
\begin{equation}
\label{eq:app-subprob}
    \begin{aligned}
        \bm{\theta}_{k+1} = &\arg\max_{\bm{\theta}\in\Theta} \quad  \left(\bm{\theta} - \bm{\theta}_{k}\right)^\top \nabla_{\bm{\theta}} \mathcal J_R(\bm{\theta}_k)
        \\
        \text{s.t.} \quad   & \mathcal J_C(\bm{\theta}_k) + \left(\bm{\theta} - \bm{\theta}_{k}\right)^\top \nabla_{\bm{\theta}} \mathcal J_C(\bm{\theta}_k) \leq b
        \\
                            & \left(\bm{\theta} - \bm{\theta}_{k}\right)^\top\left(\bm{\theta} - \bm{\theta}_{k}\right) \leq \hat\delta,
    \end{aligned}
\end{equation}

To simplify the expression, we introduce some new notations. For the $k^{th}$ iteration, we denote the gradient of the objective as $\bm{g}_k\triangleq\mathcal{J}_R(\bm\theta_k)$, the gradient of the constraint as $\bm{q}_k \triangleq \mathcal{J}_C(\bm\theta_k)$, and define $c_k \triangleq \mathcal{J}_C(\bm{\theta}_k) - b$, $\bm{\delta}_{k} \triangleq \bm{\theta}_{k+1} - \bm{\theta}_k$. Thus, we rewrite the sub-problem as
\begin{equation}\label{eq:app-rewrite-subprob}    \bm\delta_{k}=\arg\max_{\bm\delta}\bm{g}_k^\top\bm{\delta},\quad\text{s.t.}~c_k+\bm{q}_{k}^\top\bm\delta\leq0,\quad\bm{\delta}^\top\bm{\delta}\leq\hat\delta
\end{equation}

\subsubsection{Solvability Conditions for the Sub-problem} \label{app:solvability}

Since the sub-problem~\eqref{eq:app-rewrite-subprob} may have no solution, we first discuss the conditions under which this problem is solvable. The following theorem holds:
\CondOfInfea*
\begin{proof}
Denoting the boundary of the feasible region as $\mathring{\Delta}^C_k = \left\{\bm\delta\mid c_k+\bm{q}^\top_k\bm\delta=0\right\}$, let $\mathring{\bm{\delta}}_k$ be the smallest update that places the next policy $\bm\theta_{k+1}$ on the boundary of the feasible region, namely
\begin{equation}\label{eq:feasible_boundary}
    \mathring{\bm{\delta}}_k = \arg\min_{\bm\delta} \bm\delta^\top\bm\delta,\quad\text{s.t.}~c_k+\bm{q}_{k}^\top\bm\delta=0.
\end{equation}
To derive the analytical form of $\mathring{\bm{\delta}}_k$, we introduce the dual variable $\lambda\geq0$ to construct the Lagrangian dual function of the problem~\eqref{eq:feasible_boundary}:
\begin{equation}
    L(\bm\delta,\lambda) = \bm\delta^\top\bm\delta+\lambda\left(c_k+\bm{q}_{k}^\top\bm\delta\right).
\end{equation}
Since both the objective function and the constraint function in the problem \eqref{eq:feasible_boundary} are convex, the optimal solution $(\mathring{\bm{\delta}}_k,\lambda^*)$ should satisfy the Karush-Kuhn-Tucker (KKT) conditions:
\begin{equation}
\begin{aligned}
    &\nabla_{\bm\delta}L(\bm\delta,\lambda) = \bm\delta + \lambda\bm{q}_k=0
    \\
    &c_k+\bm{q}^\top_k\bm\delta = 0
\end{aligned}
\end{equation}
This can be solved to obtain $\lambda^*=c_k/\bm{q}^\top_k\bm{q}$ and $\mathring{\bm{\delta}}_k=-c_k\bm{q}_k/\bm{q}^\top_k\bm{q}_k$.

$\Longrightarrow:$ if $c_k^2 / \bm{q}^\top_k\bm{q} - \hat\delta > 0$ and $c_k>0$,
\begin{equation}
\mathring{\bm\delta}_k^\top\mathring{\bm\delta}_k = \frac{c_k^2}{\bm{q}^\top_k\bm{q}_k} > \hat{\delta}
\end{equation}
Denoting the trust region of the $k^{th}$ iteration as $\Delta_k=\left\{ \bm{\delta} \in \Theta \mid \|\bm\delta\|^2 \leq \hat\delta \right\}$.
For all $\bm\delta\in\mathring\Delta^C_k$ we have $\bm\delta^\top\bm\delta\geq\mathring{\bm\delta}_k^\top\mathring{\bm\delta}_k>\hat{\delta}$. Thus, $\forall \bm\delta\in \mathring\Delta^C_k, \delta\notin\Delta_k$, namely, the trust region $\Delta_k$ does not intersect with the boundary plane $\mathring\Delta^C_k$.

Let the feasible region for the $k^{th}$ iteration be defined as $\Delta^C_k = \left\{ \bm{\delta} \in \Theta \mid c_k + \bm{q}^\top_k\bm\delta \leq 0 \right\}$, which is a half-space determined by the plane $\mathring\Delta^C_k$. Since the trust region $\Delta_k$ does not intersect with the plane $\mathring\Delta^C_k$, only two scenarios are possible, namely, either $\Delta_k \not\subset \Delta^C_k$ or $\Delta_k \subset\Delta^C_k$.
When $c_k > 0$, $\bm0\in \Delta_k$ and $\bm0\notin \Delta^C_k$. Assuming $\Delta_k \subset \Delta^C_k$, then for all $\bm{\delta} \in \Delta_k$, it follows that $\bm{\delta} \in \Delta^C$. This leads to a contradiction, hence $\Delta_k \not\subset \Delta^C_k$. In other words, the sub-problem expressed in Equation \eqref{eq:app-rewrite-subprob} is unsolvable.

$\Longleftarrow:$ If the sub-problem expressed in Equation \eqref{eq:app-rewrite-subprob} is unsolvable, it follows that $\Delta_k \not\subset \Delta^C_k$ and $\Delta_k \cap \mathring\Delta^C_k = \emptyset$.

Given that $\Delta_k \not\subset \Delta^C_k$ and considering $\bm0\in\Delta_k$ but $\bm0\notin\Delta^C_k$, it is deduced that $c_k + \bm{q}^\top_k\bm0=c_k>0$.
Furthermore, the condition $\Delta_k \cap \mathring\Delta^C_k = \emptyset$ implies that for any $\bm\delta \in \mathring\Delta^C_k$, the relation $\bm\delta^\top\bm\delta>\hat{\delta}$ holds true. Consequently, for $\mathring{\bm\delta}_k\in\mathring\Delta^C_k$, it is observed that $\mathring{\bm\delta}_k^\top\mathring{\bm\delta}_k=c^2_k/\bm{q}^\top_k\bm{q}_k>\hat{\delta}$.
Thus, it is concluded that $c^2_k/\bm{q}^\top_k\bm{q}_k-\hat{\delta}>0$ and $c_k>0$.
\end{proof}

In addition, the following corollary holds:
\AllFea*
\begin{proof}
The proof is analogous to that of Theorem~\ref{the:cond_of_sub}. The condition $c_k^2 / \bm{q}_k^\top\bm{q}_k - \hat\delta > 0$ is necessary and sufficient for the trust region $\Delta_k$ not to intersect with the boundary of the half-space $\mathring\Delta^C_k$. Since the trust region $\Delta_k$ does not intersect with the plane $\mathring\Delta^C_k$, only two scenarios are possible: either $\Delta_k \not\subset \Delta^C_k$ or $\Delta_k \subset \Delta^C_k$. Moreover, $c_k \leq 0$ is the necessary and sufficient condition for $\Delta_k \subset \Delta^C_k$ under these circumstances.
\end{proof}

\begin{figure}[t]
    \centering
    \includegraphics[width=0.5\linewidth]{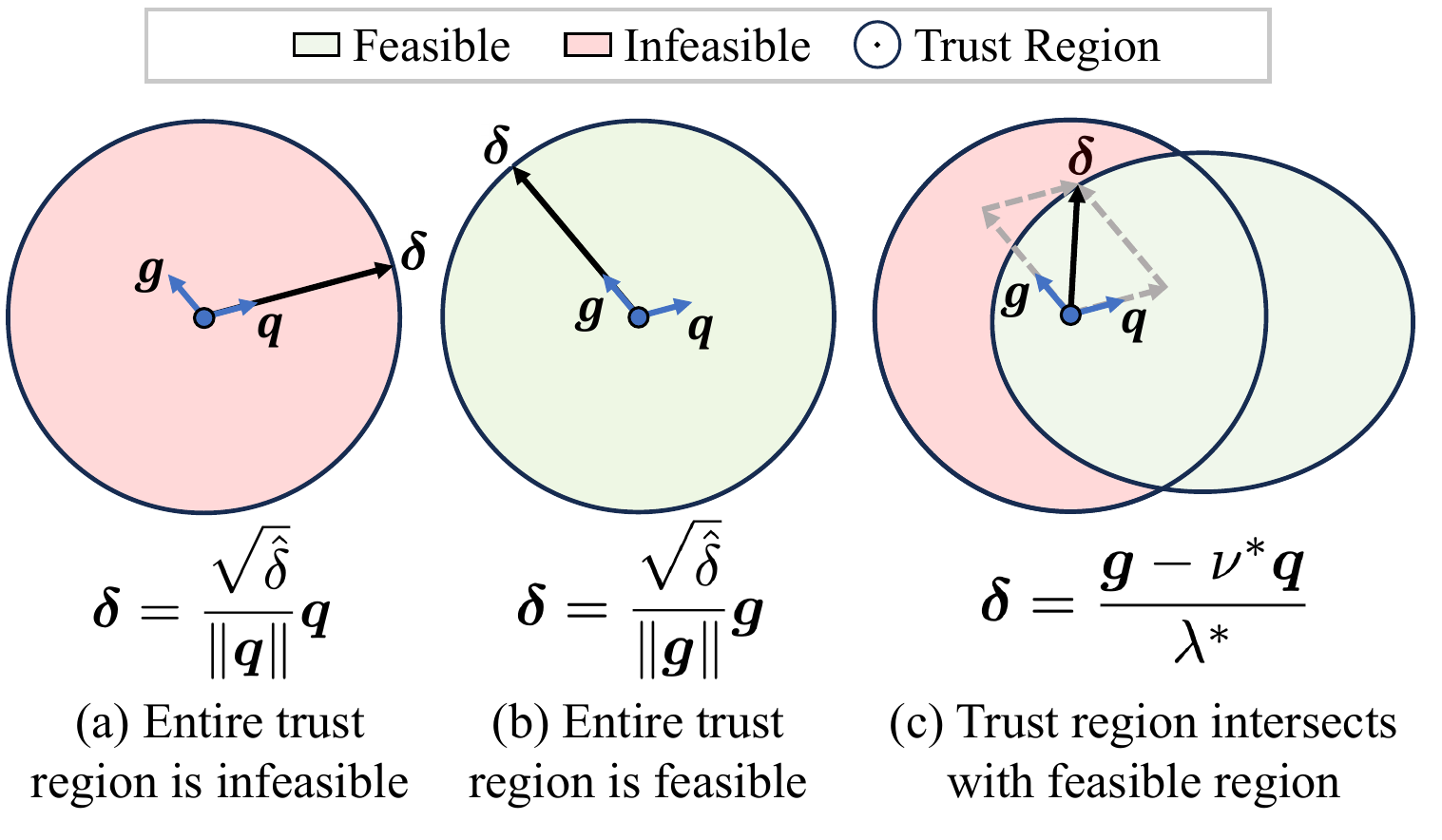}
    % \vspace{-0.5cm}
    \caption{The computational relationship between the policy update $\bm{\delta}$, the gradient of the objective function $\bm{g}$, and the gradient of the constraint function $\bm{g}$ varies in three scenarios.}
    \label{fig:three-condition}
\end{figure}

\subsubsection{the Solution to the Sub-problem}\label{app:solution-to-subprob}

Based on Theorem~\ref{the:cond_of_sub} and Corollary~\ref{the:all_feasible}, solving the sub-problem within the trust region, as defined in Equation~\ref{eq:app-rewrite-subprob}, should be discussed in three scenarios: (a) when the entire trust region is infeasible, (b) when the entire trust region is feasible, and (c) when the trust region is partially feasible, as shown in Figure~\ref{fig:three-condition}.

\begin{enumerate}[label=(\alph*)]
    \item If $c^2/{\bm{q}^\top\bm{q}}-\hat\delta>0$ and $c>0$, the entire trust region is infeasible. We update the policy along the direction of the steepest descent of the constraint function, namely,
    \begin{equation}
        \bm{\theta}_{k+1} = \bm{\theta}_{k} - \frac{\sqrt{\hat\delta}}{\|\bm{q}_k\|}\bm{q}_k.
    \end{equation}
    \item If $c^2/{\bm{q}^\top\bm{q}}-\hat\delta>0$ and $c\leq0$, the trust region lies entirely within the constraint-satisfying half-space. We directly update the step length to $\sqrt{\hat\delta}$ along the direction of the steepest ascent of the objective function, represented by the gradient vector $\bm{g}$, namely,
    \begin{equation}
        \bm{\theta}_{k+1} = \bm{\theta}_{k} + \frac{\sqrt{\hat\delta}}{\|\bm{g}_k\|}\bm{g}_k.
    \end{equation}
    \item  In other cases, the trust region partially intersects with the feasible region, so we have to solve the constrained quadratic sub-problem~\eqref{eq:app-rewrite-subprob}.
\end{enumerate}

The sub-problem~\eqref{eq:app-rewrite-subprob} is characterized as convex.
When the trust region is partially feasible, the existence of at least one strictly feasible point is guaranteed.
Consequently, strong duality holds as the conditions of Slater's theorem are fulfilled.
By introducing two dual variables $\lambda$ and $\nu$, we construct the dual function of the primal problem as follows:
\begin{equation}\label{eq:dual_fn}
    L(\bm\delta, \lambda, \nu) = -\bm{g}_k^\top\bm{\delta} + \nu\left(c_k + \bm{q}_k^\top\bm{\delta}\right) + \frac{\lambda}{2}\left(\bm{\delta}^\top\bm{\delta} - \hat\delta\right)
\end{equation}
Since the sub-problem~\eqref{eq:app-rewrite-subprob} satisfies the strong duality condition, any pair of primal and dual optimal points $(\bm\delta^*,\lambda^*,\nu^*)$ must satisfy the KKT conditions, namely,
\begin{align}
    &\nabla_{\bm\delta}L(\bm\delta,\lambda,\nu) = -\bm{g}_k + \nu \bm{q}_k + \lambda\bm\delta = 0,
    \label{eq:kkt_1}
    \\
    &\nu\left(c_k + \bm{q}_k^\top\bm{\delta}\right) = 0,
    \label{eq:kkt_2}
    \\
    &\lambda\left(\bm{\delta}^\top\bm{\delta} - \hat\delta\right) = 0,
    \label{eq:kkt_3}
    \\
    &c_k+\bm{q}^\top_k\bm\delta\leq0,\quad\bm\delta^\top\bm\delta - \hat\delta \leq0,\quad\nu \geq 0, \quad\lambda \geq 0
\end{align}

\begin{algorithm}[t]
  \caption{Dual variable solver in the case of solvability.}
  \label{alg:calc_dual_var}
  \begin{algorithmic}
    \STATE \textbf{Input:} the objective gradient $\bm{g}_k$, the constraint gradient $\bm{q}_k$, the constraint satisfaction $c_k$, and the trust region radius $\hat\delta$.
    \STATE Compute the $r_k = \bm{g}_k^\top \bm{g}_k$, $s_k = \bm{g}_k^\top \bm{q}_k$, and $ t_k = \bm{q}_k^\top \bm{q}_k$.
    \STATE Compute $\Lambda_a=\{\lambda | \lambda c_k + s_k  > 0, \lambda \geq 0\}$ and $\Lambda_b=\{\lambda | \lambda c_k + s_k  \leq 0, \lambda \geq 0\}$.
    \STATE Compute $\lambda_a = \text{Proj}\left(\sqrt{\frac{r_k-{s_k^2}/{t_k}}{\hat\delta-c_k^2/t_k}}, \Lambda_a\right)$.
    \STATE Compute $\lambda_b = \text{Proj}\left(\sqrt{\frac{r_k}{\hat\delta}}, \Lambda_b \right)$.
    \STATE Compute $L_a(\lambda_a)=\frac{1}{2\lambda_a}\left(\frac{s_k^2}{t_k}-r_k\right)+\frac{\lambda_a}{2}\left(\frac{c_k^2}{t_k}-\hat\delta\right)+\frac{s_kc_k}{t_k}$.
    \STATE Compute $L_b(\lambda_b) = -\frac{1}{2}\left(\frac{r_k}{\lambda_b}+\lambda_b \hat\delta\right)$.
    \IF {$L_a(\lambda_a) > L_b(\lambda_b)$}
    \STATE $\lambda_k^* = \lambda_a$;
    \ELSE
    \STATE $\lambda_k^* = \lambda_b$.
    \ENDIF
    \STATE Compute $\nu_k^* = \max\left(\frac{\lambda_k^* c_k + s_k}{t_k},~0\right)$.
    \STATE \textbf{Output:} $\lambda_k^*,\nu_k^*$.
  \end{algorithmic}
\end{algorithm}

Based on Equation \eqref{eq:kkt_1}, we derive that
\begin{equation}
\label{eq:delta_star}
\bm\delta^*=\frac{1}{\lambda}\left(\bm{g}_k-\nu\bm{q}_k\right).
\end{equation}
Substituting this into Equations \eqref{eq:kkt_2}, we obtain
\begin{equation}
    \nu\left(\lambda c_k + \bm{q}^\top_k\bm{g} - \nu \bm{q}^\top\bm{q}\right) = 0
\end{equation}
To simplify the expression, we denote that $r_k \triangleq \bm{g}_k^\top \bm{g}_k$, $s_k \triangleq \bm{g}_k^\top \bm{q}_k$, and $ t_k \triangleq \bm{q}_k^\top \bm{q}_k$.
Thus, it follows that
\begin{equation}
\label{eq:nu_star}
    \nu = \begin{cases}
        \frac{\lambda c_k + \bm{q}^\top_k\bm{g}_k}{\bm{q}_k^\top\bm{q}_k}, & \text{if } \lambda c_k + \bm{q}^\top_k\bm{g}_k > 0;\\
        0, & \text{if } \lambda c_k + \bm{q}^\top_k\bm{g}_k \leq 0.
    \end{cases}
    =
    \begin{cases}
        \frac{\lambda c_k + s_k}{t_k}, & \text{if } \lambda c_k + s_k > 0;\\
        0, & \text{if } \lambda c_k + s_k \leq 0.
    \end{cases}
\end{equation}
Substituting Equation~\eqref{eq:delta_star} and Equation~\eqref{eq:nu_star} into Equation~\eqref{eq:kkt_3}, we obtain
\begin{equation}
\begin{cases}
    r_k - 2s_k\left(\frac{\lambda c_k + s_k}{t_k}\right) + t_k \left(\frac{\lambda c_k + s_k}{t_k}\right)^2 = \lambda^2 \hat{\delta}, & \lambda \in \Lambda_a;
    \\
    r_k = \lambda^2\hat\delta, & \lambda \in \Lambda_b.
\end{cases}
\end{equation}
where $\Lambda_a \doteq \{\lambda | \lambda c_k + s_k  > 0, \lambda \geq 0\}$, $\Lambda_b \doteq \{\lambda | \lambda c_k + s_k \leq 0, \lambda \geq 0\}$. It is noted that when $c_k < 0$, $\Lambda_a=[0,-s_k/c_k)$ and $\Lambda_b=[-s_k/c_k, \infty)$; when $c > 0$, $\Lambda_a=[-s_k/c_k, \infty)$ and $\Lambda_b=[0, -s_k/c_k)$.
Additionally, the inequality $r_k - s_k^2/t_k > 0$ is satisfied, which is a consequence of the Cauchy-Schwarz inequality, as it implies $\|\bm{g}_k\|_2^2\|\bm{q}_k\|_2^2\geq(\bm{g}_k^\top\bm{q}_k)^2$.
Therefore, the following solution holds:
\begin{equation}
  \label{eq:lambda1}
  \lambda^*=\left\{\begin{array}{ll}
    \lambda_a\triangleq \text{Proj}\left(\sqrt{\frac{r_k-{s_k^2}/{t_k}}{\hat\delta-c_k^2/t_k}}, \Lambda_a\right),
    & L_a(\lambda_a) > L_b(\lambda_b);
    \\
    \lambda_b\triangleq \text{Proj}\left(\sqrt{\frac{r_k}{\hat\delta}}, \Lambda_b \right),
    & \text{otherwise}.
  \end{array}\right.
\end{equation}
In above context, $ L_a $ and $ L_b $ are the results obtained by substituting equations \eqref{eq:delta_star} and \eqref{eq:nu_star} into dual function \eqref{eq:dual_fn}, namely,
\begin{equation}
\left\{\begin{array}{ll}
    L_a(\lambda)\triangleq\frac{1}{2\lambda}\left(\frac{s_k^2}{t_k}-r_k\right)+\frac{\lambda}{2}\left(\frac{c_k^2}{t_k}-\hat\delta\right)+\frac{s_kc_k}{t_k},
    & \lambda \in \Lambda_a;
    \\
    L_b(\lambda)\triangleq-\frac{1}{2}\left(\frac{r_k}{\lambda}+\lambda \hat\delta\right),
    & \lambda \in \Lambda_b.
\end{array}\right.
\end{equation}
Since the dual variable $\lambda$ is used to maximize the dual function \eqref{eq:dual_fn}, we can determine which one of $\lambda_a$ and $\lambda_b$ is the optimal solution by comparing $ L_a(\lambda_a) $ and $ L_b(\lambda_b) $.

Consequently, the optimal $\nu$ can be expressed in terms of $\lambda^*$ as follows:
\begin{equation}
  \label{eq:nu1}
  \nu^* = \left(\frac{\lambda^* c_k + s_k}{t_k}\right)_+.
\end{equation}
The pseudo-code for computing the optimal $(\lambda^*_k, \nu^*_k)$ is provided in Algorithm~\ref{alg:calc_dual_var}.

Then, we update the policy as
\begin{equation}
  \bm{\theta}_{k+1} = \bm{\theta}_{k} + \frac{\bm{g}_k-\nu^*\bm{q}_k}{\lambda^*}.
\end{equation}

\subsection{Worst-Case Analysis}\label{app:convergence}

In the subsequent corollary, we detail the bounds for both performance improvement and constraint violations under the worst-case conditions following the policy update from resolving the sub-problem~\eqref{eq:app-subprob}.
\WorstCase*
\begin{proof}
Let the feasible region of the sub-problem be denoted as
\begin{equation}
    \Theta^C_k=\left\{\bm{\theta}\in\Theta\mid\mathcal J_C(\bm{\theta}_k) + \left(\bm{\theta} - \bm{\theta}_{k}\right)^\top \nabla_{\bm{\theta}} \mathcal J_C(\bm{\theta}_k) \leq b,\quad\|\bm{\theta} - \bm{\theta}_k\|^2 \leq \hat\delta\right\}.
\end{equation}
Furthermore, because
\begin{equation}
\bm{\theta}_{k+1} = \arg\max_{\bm{\theta}\in\Theta^C_k}\left(\bm{\theta} - \bm{\theta}_{k}\right)^\top \nabla_{\bm{\theta}} \mathcal J_R(\bm{\theta}_k),
\end{equation}
it follows that
\begin{equation}
    \forall \bm{\theta}\in\Theta^C_k,\quad\left(\bm{\theta}_{k+1} - \bm{\theta}_{k}\right)^\top \nabla_{\bm{\theta}} \mathcal J_R(\bm{\theta}_k)\geq \left(\bm{\theta} - \bm{\theta}_{k}\right)^\top \nabla_{\bm{\theta}} \mathcal J_R(\bm{\theta}_k).
\end{equation}
Since $\bm{\theta}_k$ is feasible, it follows that $\bm{\theta}_k \in \Theta^C_k$. Consequently, we have
\begin{equation}
\left(\bm{\theta}_{k+1} - \bm{\theta}_{k}\right)^\top \nabla_{\bm{\theta}} \mathcal J_R(\bm{\theta}_k)\geq \left(\bm{\theta}_k - \bm{\theta}_{k}\right)^\top \nabla_{\bm{\theta}} \mathcal J_R(\bm{\theta}_k)=0.
\label{eq:wc36}
\end{equation}
Further, we have $\exists t_k\in(0,1)$ such that
\begin{equation}
\begin{aligned}
\label{eq:wc37}
    \mathcal J_R (\bm{\theta}_{k+1}) - \mathcal J_R (\bm{\theta}_k)
    &= (\bm\theta_{k+1}-\bm\theta_k)^\top\nabla_{\bm{\theta}}\mathcal{J}_R(\bm{\theta}_k)
    + \frac12(\bm\theta_{k+1}-\bm\theta_k)^\top\nabla^2_{\bm{\theta}}\mathcal{J}_R(\bm{\theta}_k+t_k(\bm\theta_{k+1}-\bm\theta_k))(\bm\theta_{k+1}-\bm\theta_k)
    \\
    & \geq -\frac12\Big|(\bm\theta_{k+1}-\bm\theta_k)^\top\nabla^2_{\bm{\theta}}\mathcal{J}_R(\bm{\theta}_k+t_k(\bm\theta_{k+1}-\bm\theta_k))(\bm\theta_{k+1}-\bm\theta_k)\Big|
    \\
    &\geq-\frac12\epsilon^R_k\|\bm\theta_{k+1}-\bm\theta_k\|^2_2
    \\
    &\geq-\frac12\epsilon^R_k\hat{\delta}
\end{aligned}
\end{equation}
where
$\epsilon^R_k = \max_{t\in(0,1)}\left|\nabla_{\bm{\theta}}^2 \mathcal J_R(\bm{\theta}_k+t(\bm{\theta}_{k+1}-\bm{\theta}_k))\right|.$

On the other hand, since $\bm\theta_{k+1}$ is feasible, it follows that
\begin{equation}
    \mathcal J_C(\bm{\theta}_k) + \left(\bm{\theta}_{k+1} - \bm{\theta}_{k}\right)^\top \nabla_{\bm{\theta}} \mathcal J_C(\bm{\theta}_k) \leq b
\end{equation}
Furthermore, we have
\begin{equation}
\begin{aligned}
\label{eq:wc39}
    \mathcal{J}_C(\bm\theta_{k+1}) &= \mathcal{J}_C(\bm\theta_{k}) + (\bm\theta_{k+1}-\bm\theta_k)^\top\nabla_{\bm{\theta}}\mathcal{J}_C(\bm{\theta}_k)
    + \frac12(\bm\theta_{k+1}-\bm\theta_k)^\top\nabla^2_{\bm{\theta}}\mathcal{J}_C(\bm{\theta}_k+t_k(\bm\theta_{k+1}-\bm\theta_k))(\bm\theta_{k+1}-\bm\theta_k)
    \\
    &\leq b + \frac12(\bm\theta_{k+1}-\bm\theta_k)^\top\nabla^2_{\bm{\theta}}\mathcal{J}_C(\bm{\theta}_k+t_k(\bm\theta_{k+1}-\bm\theta_k))(\bm\theta_{k+1}-\bm\theta_k)
    \\
    &\leq b + \frac12\Big|(\bm\theta_{k+1}-\bm\theta_k)^\top\nabla^2_{\bm{\theta}}\mathcal{J}_C(\bm{\theta}_k+t_k(\bm\theta_{k+1}-\bm\theta_k))(\bm\theta_{k+1}-\bm\theta_k)\Big|
    \\
    &\leq b + \frac12\epsilon^C_k\left\|\bm\theta_{k+1}-\bm\theta_k\right\|^2_2
    \\
    &\leq b + \frac12\epsilon^C_k\hat{\delta}
\end{aligned}
\end{equation}
where
$\epsilon^C_k = \max_{t\in(0,1)}\left|\nabla_{\bm{\theta}}^2 \mathcal J_C(\bm{\theta}_k+t(\bm{\theta}_{k+1}-\bm{\theta}_k))\right|.$

By combining Equation~\eqref{eq:wc37} and Equation~\eqref{eq:wc39}, Corollary \ref{the:cgpo_wc} is thus proven.
\end{proof}

\begin{figure}[ht]
  \centering
  \includegraphics[width=0.95\textwidth]{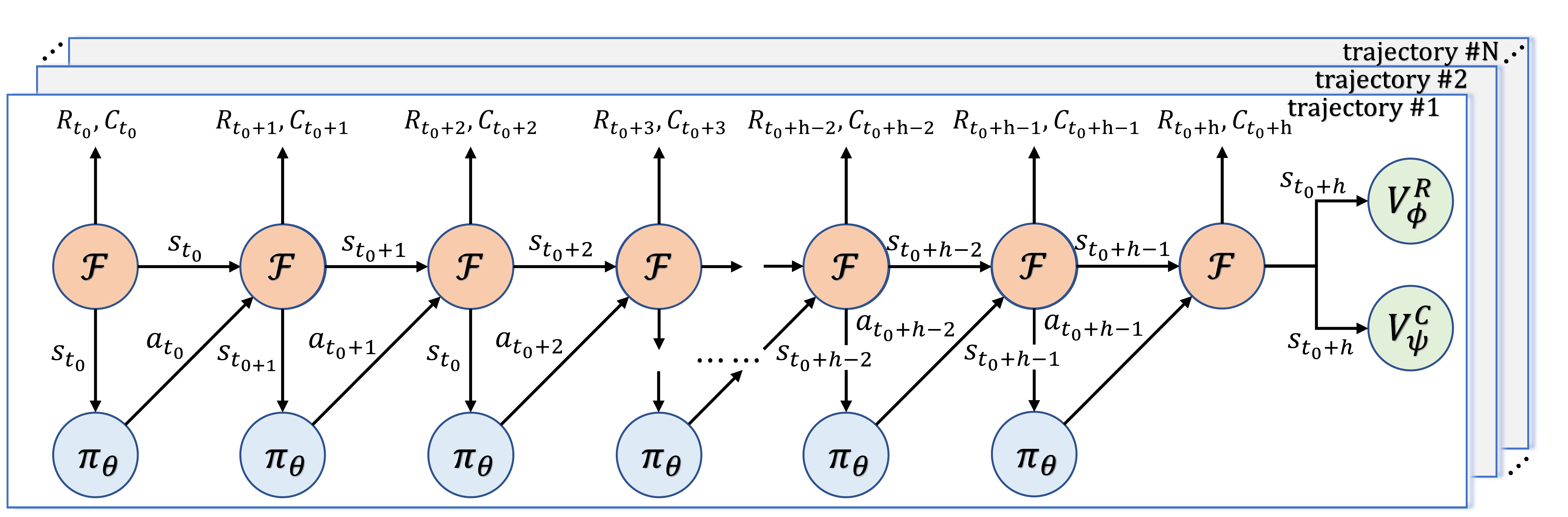}
  \caption{Computation graph of CGPO.}
  \label{fig:app-loss}
\end{figure}

\subsection{Gradient Calculation} \label{app:grad_calc}

Given our focus on employing gradients to solve Safe RL problems, our algorithm can incorporate any differentiable method for computing gradients $\bm{g_k}$ and $\bm{q_k}$, such as BPTT~\citep{mozer2013focused} and SHAC~\citep{Jie2022shac}. Here, we give an example of implementation using SHAC approach.

We compute the gradients of the objective function $\mathcal{J}_R$ and constraint functions $\mathcal{J}_C$ concerning policy $\theta$ using a methodology similar to SHAC \citep{Jie2022shac}.
While training the policy network, we concurrently train a reward critic network $V^R_{\bm{\phi}}$ and a cost one $V^C_{\bm{\psi}}$. We divide the entire trajectory into short-horizon sub-windows, each with a length of $h$. Subsequently, the gradients of the objective function (constraint function) are calculated by multi-step rewards (costs) within sub-windows plus a final value estimation from the corresponding learned critics. Specifically, considering the N-trajectory collection $\{\tau_i\}_{i=1}^N$ obtained from a given policy $\bm{a}_t = \pi_{\bm{\theta}}(\bm{s}_t)$ and  differentiable simulator $\bm{s}_{t+1}=\mathcal F(\bm{s}_t,\bm{a}_t)$, we employ the following two loss functions to compute gradients on a short-horizon sub-window:

\begin{equation}
  \label{eq:JRLoss}
  \mathcal{L}_R(\bm{\theta}) = \frac{1}{N}\sum_{i = 1}^{N}\left[\left(\sum_{t = t_0}^{t_0 + h-1}r(\mathbf{s}^i_t, \mathbf{a}^i_t)\right) + V^R_\phi(\mathbf{s}^i_{t_0 + h})\right],
\end{equation}
\begin{equation}
  \label{eq:JCLoss}
  \mathcal{L}_C(\bm{\theta}) = \frac{1}{N}\sum_{i = 1}^{N}\left[\left(\sum_{t = t_0}^{t_0 + h-1}c(\mathbf{s}^i_t, \mathbf{a}^i_t)\right) +  V^C_\psi(\mathbf{s}^i_{t_0 + h})\right],
\end{equation}
where $t_0$ is the starting time step of a sub-window.
To compute the gradients $\frac{\partial \mathcal{L}_R(\bm{\theta})}{\partial \bm{\theta}}$ and $\frac{\partial \mathcal{L}_C(\bm{\theta})}{\partial \bm{\theta}}$, we treat the simulator $\mathcal{F}$ as a differentiable layer and incorporate it into the overall computational graph, as illustrated in Figure \ref{fig:app-loss}.
The gradients are then computed by back-propagation through time (BPTT) \citep{bptt}, with the replacement of the final time step by values $ V^R_\psi(\mathbf{s}^i_{t_0 + h})$, $ V^C_\psi(\mathbf{s}^i_{t_0 + h})$, and with truncation at the beginning of each sub-window.

When performing specific gradient calculations, we need to start from the last time step $t_0+h$ of the sub-window for each trajectory $\tau_i$:
\begin{align}
    \frac{\partial \mathcal{L}_R(\theta)}{\partial s^i_{t_0+h}}=\frac{1}{N}\frac{\partial V^R_\phi(s^i_{t_0+h})}{\partial s^i_{t_0+h}},\\
    \frac{\partial \mathcal{L}_C(\theta)}{\partial s^i_{t_0+h}}=\frac{1}{N}\frac{\partial V^C_\psi(s^i_{t_0+h})}{\partial s^i_{t_0+h}}.
\end{align}
Then, we can compute the adjoints in the previous steps $t_0\leq t<t_0+h$ in reverse order:
\begin{equation}
    \frac{\partial \mathcal{L}_R(\theta)}{\partial s^i_t}=\frac{1}{N}\frac{\partial R(s^i_t,a^i_t)}{\partial s^i_t}+\left(\frac{\partial \mathcal{L}_R(\theta)}{\partial s^i_{t+1}}\right)\left(\left(\frac{\partial \mathcal{F}(s^i_t, a^i_t)}{\partial s^i_t}\right)+\left(\frac{\partial \mathcal{F}(s^i_t, a^i_t)}{\partial a^i_t}\right)\left(\frac{\pi_\theta(s^i_t)}{\partial s^i_t}\right)\right)
\end{equation}
\begin{equation}
    \frac{\partial \mathcal{L}_C(\theta)}{\partial s^i_t}=\frac{1}{N}\frac{\partial C(s^i_t,a^i_t)}{\partial s^i_t}+\left(\frac{\partial \mathcal{L}_C(\theta)}{\partial s^i_{t+1}}\right)\left(\left(\frac{\partial \mathcal{F}(s^i_t, a^i_t)}{\partial s^i_t}\right)+\left(\frac{\partial \mathcal{F}(s^i_t, a^i_t)}{\partial a^i_t}\right)\left(\frac{\pi_\theta(s^i_t)}{\partial s^i_t}\right)\right)
\end{equation}
\begin{equation}
    \frac{\partial \mathcal{L}_R(\theta)}{\partial a^i_t}=\frac{1}{N}\frac{\partial R(s^i_t,a^i_t)}{\partial a^i_t} + \left(\frac{\partial \mathcal{L}_R(\theta)}{\partial s^i_{t+1}}\right)\left(\frac{\partial \mathcal{F}(s^i_t, a^i_t)}{\partial a^i_t}\right)
\end{equation}
\begin{equation}
    \frac{\partial \mathcal{L}_C(\theta)}{\partial a^i_t}=\frac{1}{N}\frac{\partial C(s^i_t,a^i_t)}{\partial a^i_t} + \left(\frac{\partial \mathcal{L}_C(\theta)}{\partial s^i_{t+1}}\right)\left(\frac{\partial \mathcal{F}(s^i_t, a^i_t)}{\partial a^i_t}\right)
\end{equation}
From all the computed adjoints, we can compute the objective loss and the constraint loss by
\begin{equation}
    \frac{\partial \mathcal{L}_R(\theta)}{\partial \theta}=\sum^N_{i=1}\sum^{t_0+h-1}_{t=t_0}\left(\frac{\partial\mathcal{L}_R(\theta)}{\partial a^i_t}\right)\left(\frac{\partial\pi_\theta(s^i_t)}{\partial \theta}\right)
\end{equation}
\begin{equation}
    \frac{\partial \mathcal{L}_C(\theta)}{\partial \theta}=\sum^N_{i=1}\sum^{t_0+h-1}_{t=t_0}\left(\frac{\partial\mathcal{L}_C(\theta)}{\partial a^i_t}\right)\left(\frac{\partial\pi_\theta(s^i_t)}{\partial \theta}\right)
\end{equation}

\subsection{The Update Method of Critic Network}\label{app:value}

As mentioned in Section 5.3, line 309, the time step $t$ is incorporated into the state as an input to the critic network, namely, $V(s_t, t)$. This value function forecasts the expected return for a future finite horizon of $T-t$ at time step $t$ and state $s$:
\begin{equation}
V_\pi(s, t) \doteq \mathbb E_\pi\left[\sum_{k=1}^{T-t}R_{t+k}\mid S_t=s\right]
\end{equation}
We define a series of n-step estimates as follows:
\begin{equation}
G_{t:t+n} \doteq R_{t+1} + R_{t+2} + \cdots + R_{t+n} + V_\pi(S_{t+n}, t+n)
\end{equation}
Specifically, in our implementation, the data $\{\tau^i_{t_0:t_h}\}^N_{i=1}$ obtained from each short horizon $h$ rollout is used to calculate $G_{t:t+n},n=1,2,\cdots,h-t$.
Given $\lambda$, an unbiased finite-horizon TD($\lambda$) estimate of $V(s^i_t,t),s^i_t\in\tau^i_{t_0:t_h}$ can be obtained by the following formula:

\begin{equation}
G^{\lambda,i}_t=(1-\lambda)\sum^{h-t-1}_{n=1}\lambda^{n-1}G^i_{t:t+n}+\lambda^{h-t-1}G^i_{t:t+h}
\end{equation}
Note that the weight decays as $n$ increases and the total summation is 1.
Thus, the updated Loss function can be defined as:
\begin{equation}
    L^{\lambda}(\phi) = \frac{1}{Nh}\sum^N_{i=1}\sum^{h-1}_{t=0}\left(G^{\lambda,i}_t - V_\phi(s^i_t, t)\right)^2
\end{equation}

\section{Zero-order Batched Gradient method vs First-order Batched Gradient method}\label{app:zobg}

Given our focus on employing gradients to solve Safe RL problems, our algorithm can incorporate any differentiable method to compute gradients $\bm{g_k}$ and $\bm{q_k}$.
Various methods are available, broadly classified into two categories: the Zero-order Batch Gradient (ZoBG) method and the First-order Batch Gradient (FoBG) method~\citep{suh2022differentiable}.

compared to zeroth-order batched gradient (ZoBG) estimation on $\{\tau_i\}^N_{i=1}$,
\begin{equation}
    \nabla_\theta\mathcal J(\theta) \simeq g_\text{ZoBG} \doteq \frac1N\sum_{i=1}^N\left[\sum^T_{t=0}\nabla_\theta\log\pi_\theta(a_t^i\mid s_t^i) R(\tau_i)\right],
\end{equation}
we find that FoBG is more suitable for the Safe RL field for the following two reasons:

\paragraph{First, using ZoBG to estimate the constraint at $\bm\theta_k +\bm\delta$ results in significant estimation errors.}
In Safe RL, the estimation of constraint condition $\mathcal J_C(\bm\theta_k +\bm\delta)\leq b$ serves as an important basis for solving the optimal feasible update $\bm\delta^*$.
Thus, inaccurate estimations can lead to suboptimal constraint satisfaction.
Intuitively, since ZoBG does not utilize the first-order gradient of environmental dynamics, it fails to capture how policy changes affect state visitation changes.
Particularly in long trajectories, this effect is amplified, resulting in greater changes to states visited later due to earlier state modifications.
However, first-order batched gradient (FoBG) estimation mitigates this amplification effect by incorporating the first-order gradient of environmental dynamics, allowing the gradient to propagate along the trajectory.

We empirically confirmed this through experiments comparing the relative errors of ZoBG and FoBG in estimating constraint functions.
These experiments were conducted in simple FunctionEnv environments, as detailed in Appendix D.1, across various trajectory lengths and update magnitudes $\hat\delta$.
Each experiment ran for 100 epochs using 5 random seeds, with the results displayed in the tables below:

\begin{table*}[ht]
\centering
\vspace{-1em}
\caption{Relative error of ZoBG in constraint function estimation in the simple FunctionEnv which is detailed in Appendix~\ref{app:simple_env}.}
\label{tab:zobg}
\resizebox{0.98\textwidth}{!}{
\begin{threeparttable}
\begin{tabular}{ccccccc}
\toprule
Traj. Len. $\rightarrow$ & 1 & 5 & 10 & 50 & 100 & 200 \\
\midrule
$\hat\delta$ = 0.0001 & 0.3387 $\pm$ 0.2420 & 0.8862 $\pm$ 0.3499 & 3.093 $\pm$ 1.394 & 15.40 $\pm$ 15.93 & 35.90 $\pm$ 29.11 & 65.19 $\pm$ 138.6 \\
$\hat\delta$ = 0.0001 & 0.6291 $\pm$ 1.4107 & 0.4401 $\pm$ 1.722 & 15.575 $\pm$ 6.813 & 32.09 $\pm$ 132.9 & 150.6 $\pm$ 68.59 & 39.17 $\pm$ 134.4 \\
$\hat\delta$ = 0.0001 & 0.3171 $\pm$ 2.039 & 8.709 $\pm$ 33.90 & 31.38 $\pm$ 139.1 & 68.71 $\pm$ 292.6 & 197.6 $\pm$ 488.8 & 66.40 $\pm$ 348.7 \\
$\hat\delta$ = 0.0001 & 1.630 $\pm$ 11.76 & 4.393 $\pm$ 17.18 & 22.56 $\pm$ 132.8 & 276.8 $\pm$ 372.7 & 202.4 $\pm$ 291.1 & 35.99 $\pm$ 181.0 \\
$\hat\delta$ = 0.0001 & 12.00 $\pm$ 26.20 & 11.51 $\pm$ 66.89 & 65.47 $\pm$ 118.9 & 94.81 $\pm$ 93.27 & 174.9 $\pm$ 515.3 & 71.00 $\pm$ 288.0 \\
\bottomrule
\end{tabular}
\end{threeparttable}
}
\end{table*}

\begin{table*}[ht]
\centering
\vspace{-1em}
\caption{Relative error of FoBG in constraint function estimation in the simple FunctionEnv which is detailed in Appendix~\ref{app:simple_env}.}
\label{tab:fobg}
\resizebox{0.98\textwidth}{!}{
\begin{threeparttable}
\begin{tabular}{ccccccc}
\toprule
Traj. Len. $\rightarrow$ & 1 & 5 & 10 & 50 & 100 & 200 \\
\midrule
$\hat\delta$ = 0.0001 & 0.0081 $\pm$ 0.0070 & 0.0033 $\pm$ 0.0029 & 0.0018 $\pm$ 0.0015 & 0.0007 $\pm$ 0.0006 & 0.0011 $\pm$ 0.0014 & 0.0015 $\pm$ 0.0024 \\
$\hat\delta$ = 0.0001 & 0.0029 $\pm$ 0.0022 & 0.0027 $\pm$ 0.0021 & 0.0027 $\pm$ 0.0022 & 0.0030 $\pm$ 0.0028 & 0.0059 $\pm$ 0.0163 & 0.0129 $\pm$ 0.0606 \\
$\hat\delta$ = 0.0001 & 0.0050 $\pm$ 0.0038 & 0.0046 $\pm$ 0.0038 & 0.0045 $\pm$ 0.0039 & 0.0056 $\pm$ 0.0064 & 0.0084 $\pm$ 0.0154 & 0.0165 $\pm$ 0.0661 \\
$\hat\delta$ = 0.0001 & 0.0283 $\pm$ 0.0295 & 0.0221 $\pm$ 0.0226 & 0.0194 $\pm$ 0.0170 & 0.0189 $\pm$ 0.0227 & 0.0224 $\pm$ 0.0327 & 0.0494 $\pm$ 0.2030 \\
$\hat\delta$ = 0.0001 & 0.1433 $\pm$ 1.1200 & 0.0371 $\pm$ 0.0502 & 0.0308 $\pm$ 0.0330 & 0.0290 $\pm$ 0.0406 & 0.0399 $\pm$ 0.0577 & 0.1045 $\pm$ 0.5258 \\
\bottomrule
\end{tabular}
\end{threeparttable}
}
\end{table*}

\begin{figure*}[ht]
    \centering
    \includegraphics[width=0.8\linewidth]{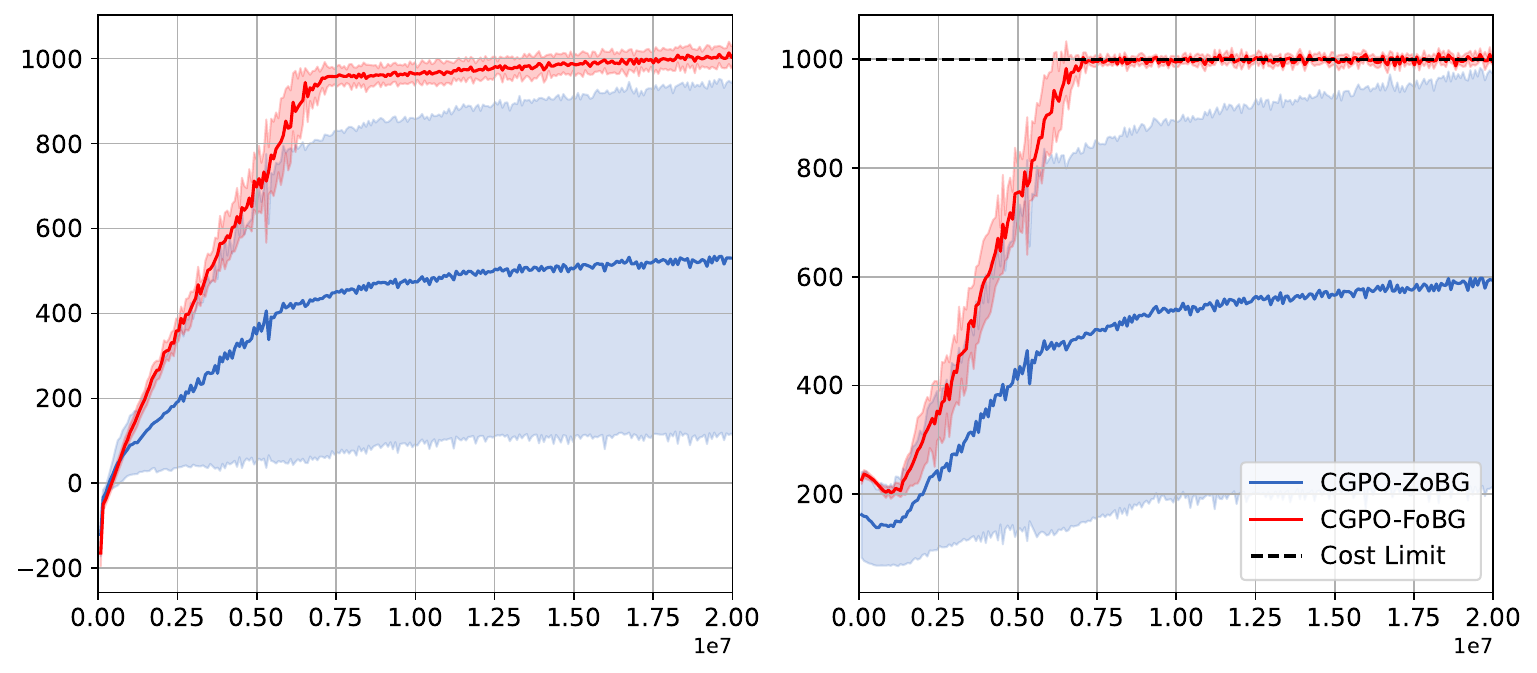}
    \vspace{-1.2em}
    \caption{Training curves of CGPO using ZoBG method and FoBG method to calculate the gradients, showing episodic return and constraint for 5 random seeds. Solid lines represent the mean, while the shaded areas indicate variance, without any smoothing to the curves.}
    \label{fig:fobg}
\end{figure*}

Table~\ref{tab:zobg} shows that ZoBG effectively estimates the constraint function only at short trajectory lengths ($\leq 5$) and small update magnitudes ($\leq 0.001$).
However, with longer trajectory lengths and larger update magnitudes, ZoBG quickly loses its ability to estimate constraints ($\text{Relative Error} \geq 1$).
In contrast, Table~\ref{tab:fobg} demonstrates that FoBG consistently estimates the constraint function accurately across all trajectory lengths and update magnitudes, indicating superior stability.
These results highlight FoBG's significant advantages over ZoBG in constraint prediction within the Safe RL domain.

\paragraph{Second, ZoBG, which employs the Monte Carlo method, exhibits significant estimation variance and poor sample efficiency.}
To validate this, we introduced a new baseline, CGPO-ZoBG, in our training, where ZoBG is utilized for gradient estimation in our algorithm.
The training curve of CGPO-ZoBG is depicted in the figure below:

Figure~\ref{fig:fobg} illustrates that CGPO-ZoBG has low and unstable update efficiency, complicating the search for the optimal feasible update. This evidence further underscores ZoBG's inadequacy for Safe RL.

In summary, despite challenges in its application, FoBG's advantages in Safe RL are more pronounced than those of ZoBG.
Furthermore, as the techniques of differentiable simulators and world models continue to evolve, FoBG's application potential will significantly expand.
Therefore, we argue that integrating differentiable RL methods into our algorithm is essential.

\section{Experiment Highlighting the Shortcomings of the ABE Method}\label{app:simple_env}

\subsection{Task Settings}

\begin{wrapfigure}{r}{0.4\textwidth}
\vspace{-2em}
\includegraphics[width=\linewidth]{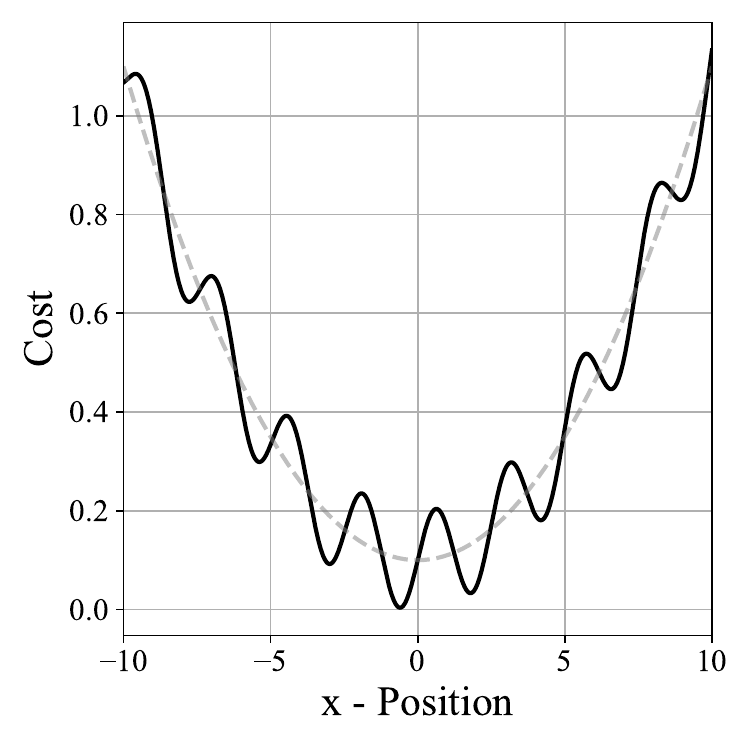}
\vspace{-2.0em}
\caption{Cost values at different x-axis positions for the agent.}
\label{fig:simple-exp-cost}
\end{wrapfigure}

To clearly demonstrate the shortcomings of the \textit{Advantage-based Estimation} (ABE) method, we designed a simple task. The goal of this task is to train an agent to move along the x-axis, receiving rewards and incurring costs based on its position on the x-axis. The specific settings of the task are as follows:

\textbf{Action Space:} The direction and step length of the agent's movement along the x-axis constitute the action space, which is a single-element vector, i.e., $a \in (-1, 1)$.

\textbf{Observation Space:} The observation data for the agent is its position on the x-axis, which is also a single-element vector, i.e., $s \in \mathbb{R}$. The task employs a simple, artificially set state transition function: $s_{t+1} = s_t + 0.2 \times a_t$.

\textbf{Reward Function and Cost Function:} Since the purpose of this simple task is to verify the accuracy of the estimation method, we set the reward function and cost function in the same form, i.e.,
\begin{equation}
    f(x) = \left(\frac{x}{10.0}\right)^2 + 0.1 + 0.1 \times \sin\left(\frac{8x}{\pi}\right)
\end{equation}
This function is sourced from~\citet{metz2021gradients}, and its visualization is shown in Figure \ref{fig:simple-exp-cost}.

\textbf{Constraint:} Since the task is solely aimed at verifying the accuracy of the estimation method, the set constraint has no physical significance. Specifically, the cumulative cost on a trajectory of length 100 for the agent should be less than 8.0:
\begin{equation}
    \mathcal{J}_C(\pi) \triangleq \sum^{99}_{t=0} f(x_t)\leq 8.0
\end{equation}
This constraint setting is consistent with many well-known Safe Benchmarks (such as Safety-Gym~\citep{ray2019benchmarking}, Bullet-Safety-Gym~\citep{gronauer2022bullet}, Safety-Gymnasium~\citep{ji2023safety}, and OmniSafe~\citep{ji2023omnisafe}), where the undiscounted cumulative cost on a finite-length trajectory must be less than a scalar threshold, namely, $\sum^{T-1}_{t=0}c_t \leq b$.

\subsection{Experimental Details}

In the training procedure, the parameterized initial strategy $\pi_{\theta_0}$ is updated iteratively, leveraging both the Advantage-based Estimation (ABE) and the Gradient-based Estimation (GBE) methods. The magnitude of each update, denoted by $\bm\delta$, remains fixed at $\|\bm\delta\|=0.01$. For every update, the value of the updated constraint function is computed separately using the ABE and GBE methods:

\textbf{ABE method:}
\begin{equation}
    \hat{\mathcal{J}}^\text{ABE}_C(\bm\theta+\bm\delta) \triangleq \mathcal{J}_C(\bm\theta) + \frac{1}{1-\gamma}\E_{\substack{s\sim d^{\pi_{\bm\theta}} \\ a\sim{\pi_{\bm\theta}}}} \left[\frac{\pi_{\bm\theta+\bm\delta}(s,a)}{\pi_{\bm\theta}(s,a)}A^{\pi_{\bm\theta}}(s,a)\right]
\end{equation}

\textbf{GBE method:}
\begin{equation}
    \hat{\mathcal J}^\text{GBE}_C(\bm{\theta}+\bm{\delta}) \triangleq \mathcal J_C (\bm{\theta}) + \bm{\delta}^\top \nabla_{\bm{\theta}} \mathcal J_C(\bm{\theta}).
\end{equation}

\textbf{Metrics:} Given that both methods aim to approximate the variation relative to $J_C (\bm{\theta})$, the relative error in the actual shift of the constraint function is employed as the evaluative metric:
\begin{equation}
    \Delta_\text{relative} = \frac{\Big|{\mathcal{J}}_C(\bm\theta+\bm\delta)-\hat{\mathcal{J}}_C(\bm\theta+\bm\delta)\Big|}{\Big|{\mathcal{J}}_C(\bm\theta+\bm\delta) - {\mathcal{J}}_C(\bm\theta)\Big|}
\end{equation}
Notably, an estimation is considered valid when the relative error stays below 1.0. Conversely, an estimation loses its scalar referential value when the relative error exceeds 1.0.

\newpage
\section{More Experimental Results}
\subsection{Main Results}\label{app:more_results}

To ensure clarity, we avoid placing numerous training curves on a single graph as this can lead to difficulties in differentiation.
Given the space constraints in the main text, the comprehensive set of training curves for all algorithms is provided exclusively in the appendix. Additionally, the baseline algorithms have been categorized into two distinct groups: Traditional Safe RL Algorithms and Lagrangian-Revised Differentiable RL Algorithms.

\begin{figure}[H]
    \centering
    \includegraphics[width=0.98\linewidth]{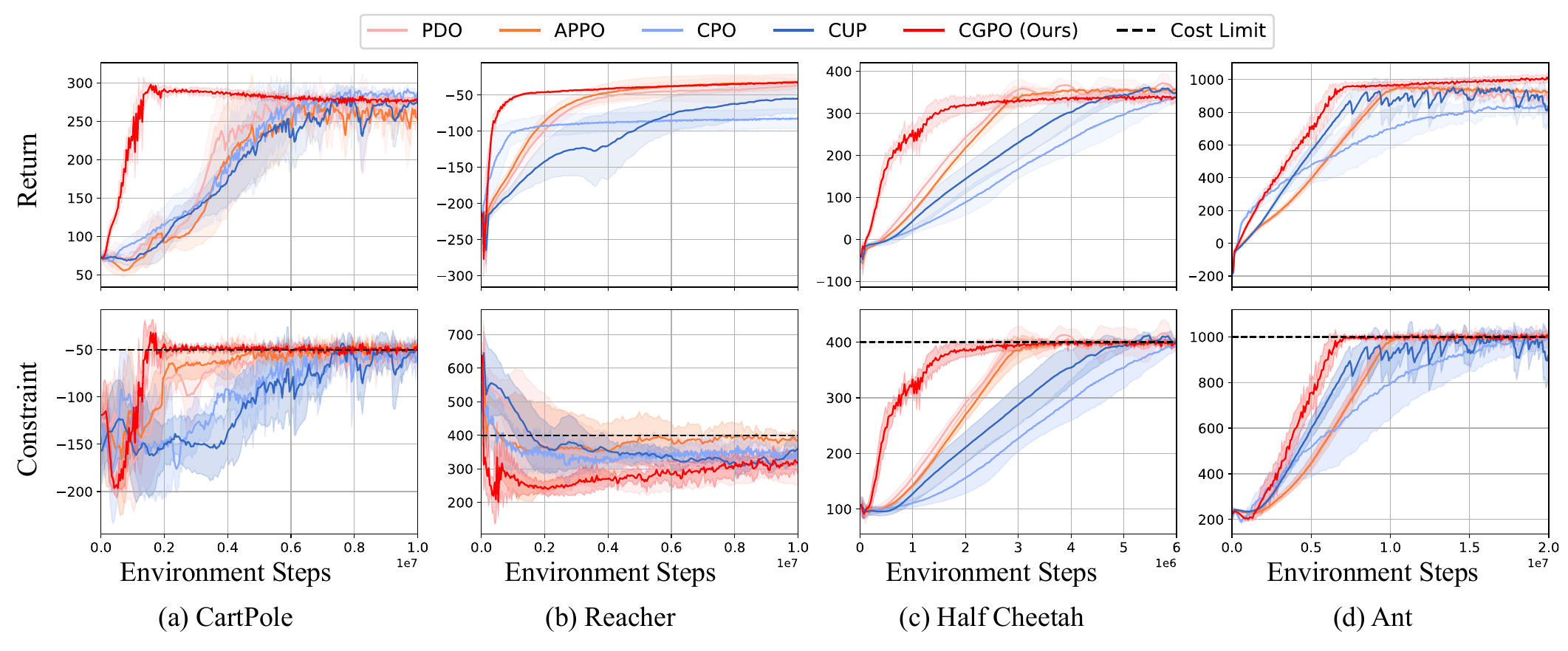}
    \vspace{-1.2em}
    \caption{Training curves compared to the Traditional Safe RL Algorithms on different tasks, showing episodic return and constraint for 5 random seeds. Solid lines represent the mean, while the shaded areas indicate variance, without any smoothing to the curves.}
    \label{fig:main_1}
\end{figure}

\begin{figure}[H]
    \centering
    \includegraphics[width=0.98\linewidth]{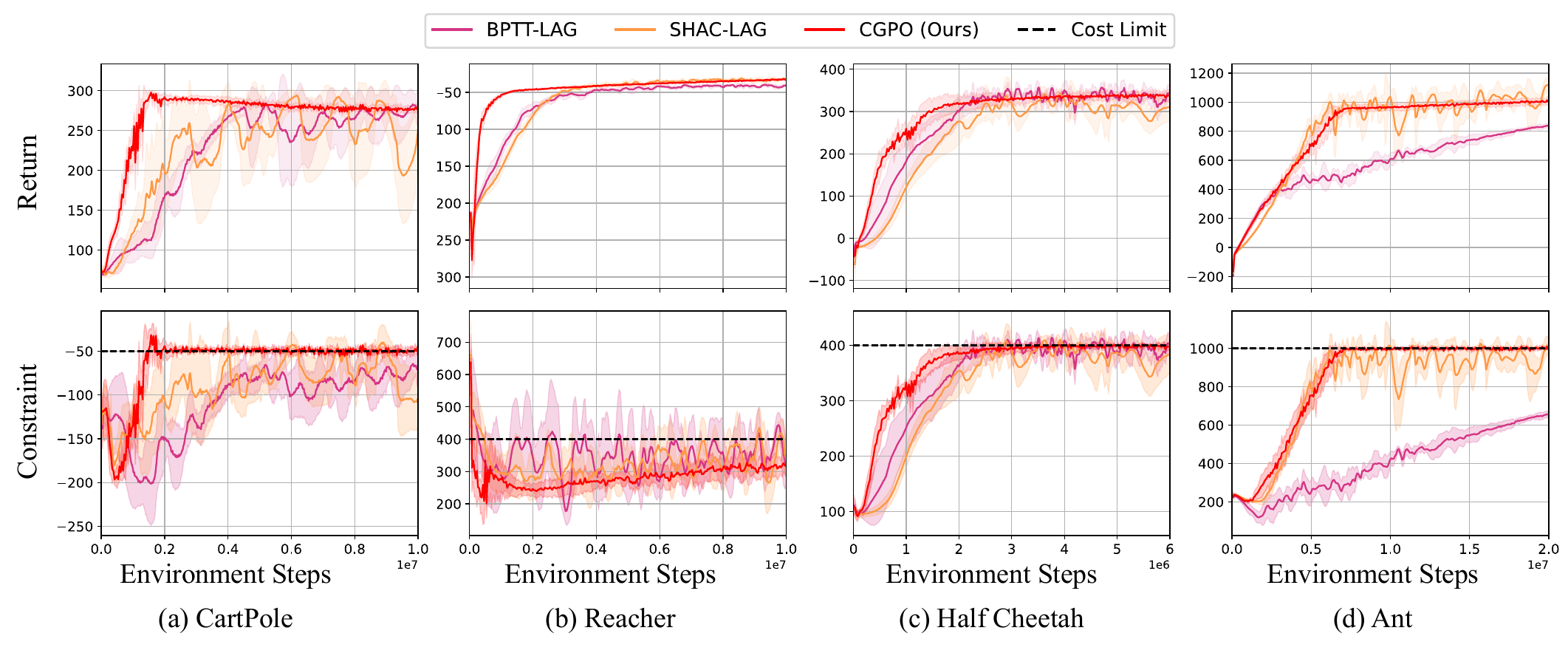}
    \vspace{-1.2em}
    \caption{Training curves compared to the Lagrangian-revised differentiable Safe RL algorithms on different tasks, showing episodic return and constraint for 5 random seeds. Solid lines represent the mean, while the shaded areas indicate variance, without any smoothing to the curves.}
    \label{fig:main_2}
\end{figure}

\subsection{More Ablation Experiments on Radius Adaptation}\label{app:adaptation}

Large changes in hyper-parameter $\hat\delta$ have some impact on the convergence speed and constraint satisfaction, as shown in Figure~\ref{fig:adaptation_params} (a). Therefore, we propose an adaptive delta setting method, which can automatically adjust the value of delta according to the current training status, further reducing the difficulty of setting this hyperparameter.
To validate the effectiveness of the adaptive method in managing hyperparameter settings, we conducted a series of ablation studies. The outcomes of these experiments are presented below:

 \begin{figure}[ht]
    \centering
    \includegraphics[width=0.98\linewidth]{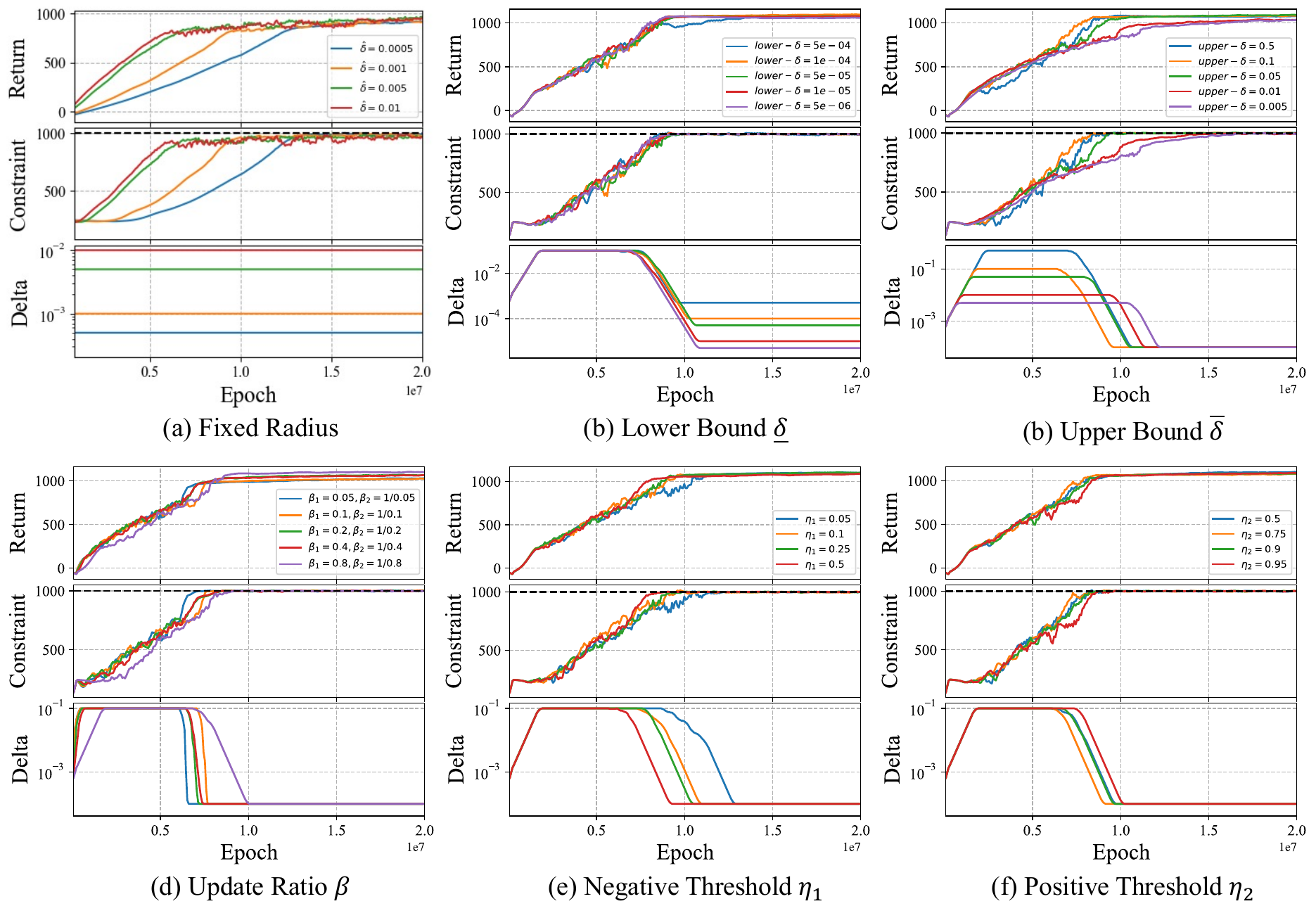}
    \vspace{-1.2em}
    \caption{Ablation studies on the hyper-parameters of the adaptive method of trust-region radius.}
    \label{fig:adaptation_params}
\end{figure}

Figure~\ref{fig:adaptation_params} (b) - (f) illustrate that the adaptive approach's hyper-parameters exhibit robustness, demonstrating its effectiveness in transforming a challenging-to-tune hyper-parameter into several more manageable hyper-parameters.

\begin{comment}
\subsection{Ablation Experiments in Infinite-horizon Discounted Formulations}

 \begin{figure}[ht]
    \centering
    \includegraphics[width=0.98\linewidth]{figures/velocity_ant.pdf}
    \vspace{-1.2em}
    \caption{Training curves of certain algorithms in Ant Task with infinite-horizon discounted constraints, showing episodic return and constraint for 5 random seeds. Solid lines
represent the mean, while the shaded areas indicate variance, without any smoothing to the curves.}
    \label{fig:infinite}
\end{figure}

As shown in Figure~\ref{fig:infinite}, it can be seen that the performance of CGPO under infinite-horizon constraint conditions is the same as that under finite-horizon constraint conditions. This indicates that our algorithm does not require a specific form of the constraint. In contrast, traditional Safe RL algorithms show better performance improvements and constraint satisfaction capabilities under infinite-horizon settings than under finite-horizon settings, especially CPO and CUP algorithms.
\end{comment}

\newpage
\section{Implementation of Differentiable Safe RL Tasks}\label{app:impl_task}

Our algorithm requires computing first-order gradients along trajectories, necessitating a guarantee of system differentiability.
Given the rapid advancement of differentiable physics engines~\citep{degrave2019differentiable,werling2021fast,xian2023fluidlab}, we selected robotic control tasks for our tests.
We developed the first Differentiable Safe RL task series, based on four robotic control tasks in the fully differentiable Brax engine~\citep{freeman2021brax}, as shown in Figure~\ref{fig:agent}
We introduced two common constraints to these four unconstrained tasks: one on position~\citep{achiam2017constrained,yang2018unified,ji2023safety} and another on velocity~\citep{zhang2020first,shi2019tbq}.
The specific constraint settings may lack sufficient physical significance, yet they suffice for algorithm testing.
This section details the settings of these tasks.

\begin{figure}[ht]
    \centering
    \includegraphics[width=0.85\linewidth]{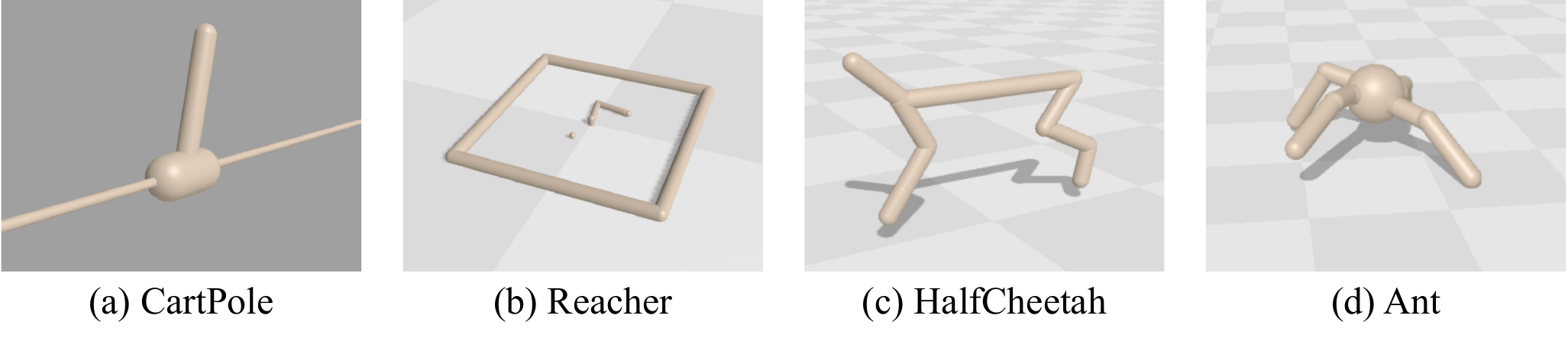}
    \vspace{-1.2em}
    \caption{Four distinct agents used in our experiments originate from Brax, a fully differentiable physics engine.}
    \label{fig:agent}
\end{figure}

\subsection{Position-constrained CartPole}

This environment features a cart that moves linearly with one end of a pole fixed to it and the other end free.
The goal is to push the cart left or right to balance the pole on its top by applying forces.
The constraint requires maintaining the cart's position at a specific distance from the origin.
Since the optimal policy involves the cart staying stationary at its initial position, the constraint introduces a contradiction between the constraints and the goal.

\textbf{Action Space:} The agent takes a 1-element vector for actions. The action space is a continuous range,  $a \in [-3,3]$. The action denotes the numerical force applied to the cart, with its magnitude indicating the force amount and its sign in the direction.

\textbf{Observation Space:} The state space includes positional values of the pendulum system's body parts, followed by the velocities of these parts, with positions listed before velocities.	The observation is a vector with a shape of $(4,)$.

\textbf{Reward Function:} Receives a reward of 1.0 at each time step, i.e., $r_t = 1.0$.

\textbf{Cost Function:} Incurs a cost equal to the negative square of the cart's position at each time step, i.e., $c_t = - \|\bm{x}_\text{pos}\|^2$.

\textbf{Constraint:} The cumulative cost on a trajectory of length $300$ for the agent should be less than $-50.0$.

\subsection{Position-constrained Reacher}

Reacher is a two-jointed robot arm.
The goal is to guide the robot's end effector (fingertip), close to a target appearing at a random location.
We set a constraint opposite to the objective function, requiring that the fingertip must not be too close to the target point.

\textbf{Action Space:} The action space is defined as a bounded box, designated as Box(-1, 1, (2,), float32).
This notation represents a two-dimensional continuous space, with each action represented by a pair $(a, b)$.
These values represent the torques' magnitudes applied to the hinge joints, enabling precise control of the system's movements.

\textbf{Observation Space:} The observation data includes key aspects: cosine and sine values of the two arms' angles, target coordinates, arms' angular velocities, and a three-dimensional vector from the target to the fingertip.
This information is compiled into a vector with a shape of $(11,)$.

\textbf{Reward Function:} At each time step, the reward function comprises two parts: a distance reward $r_\text{dist} = -\|\bm{x}_\text{pos}\|$ and a control reward $r_\text{ctrl} = -(a^2+ b^2)$, and $r_t = r_\text{dist} + \beta_\text{ctrl} \cdot r_\text{ctrl}$. Here, $\beta_\text{ctrl}$ is a pre-configured coefficient.

\textbf{Cost Function:} Incurs a cost equal to the square of the position of the fingertip at each time step, i.e., $c_t = \|\bm{x}_\text{pos}\|^2$.

\textbf{Constraint:} The cumulative cost on a trajectory of length $300$ for the agent should be less than $400.0$.

\subsection{Velocity-constrained HalfCheetah}

The HalfCheetah, a 2D robot, features 9 linked segments and 8 joints, with two joints designed to mimic paws.
The main goal is to strategically apply torque to these joints to make the cheetah sprint forward swiftly.
The agent gets a positive reward for distance moved forward and a negative reward for moving backward.
Notably, the cheetah's torso and head are static, with a torque applied only to six joints: the front and back upper legs (attached to the torso), lower legs (attached to the upper legs), and feet (attached to the lower leg).
The constraint involves requiring the velocity to be less than a threshold, as referenced from~\citet{zhang2020first}.

\textbf{Action Space:} The action space is defined as a bounded box, designated as Box(-1, 1, (6,), float32).
The agents operate using a set of six actions, each represented by a vector element.
The values for these actions fall within the range of $-1.0$ to $1.0$.

\textbf{Observation Space:}
In the state space, the cheetah's body parts' positions are detailed first, followed by their velocities, representing their rate of change.
This information is organized with all positional data listed before the velocity data.
The observation is a vector shaped as $(18,)$.

\textbf{Reward Function:} At each time step, the reward function comprises two parts: an x-axis velocity reward $r_\text{vel} = v_x$ and a control reward $r_\text{ctrl} = \|\bm{a}\|^2$, and $r_t = r_\text{dist} + \beta_\text{ctrl} \cdot r_\text{ctrl}$. Here, $\beta_\text{ctrl}$ is a pre-configured coefficient.

\textbf{Cost Function:} At each time step, the incurred cost is equal to the value of the velocity, denoted as $ c_t = \|\bm{v}\| $.

\textbf{Constraint:} The cumulative cost on a trajectory of length $200$ for the agent should be less than $400.0$. The cost thresholds are calculated using 50\% of the speed attained by an unconstrained PPO agent after training for sufficient samples~\citep{zhang2020first}.

\subsection{Velocity-constrained Ant}

The Ant is a 3D robot with a torso-like, freely rotating body and four attached legs.
Each leg comprises two segments joined by hinges.
The goal is to maneuver these legs coordinately, propelling the robot forward and to the right.
This is achieved by precisely applying forces to the eight hinges connecting the leg segments to the torso, managing a system of nine parts and eight hinges.
The constraint involves requiring the velocity to be less than a threshold, as referenced from~\citet{zhang2020first}.

\textbf{Action Space:} The action space is defined as a bounded box, designated as Box(-1, 1, (8,), float32).
The agent operates with an 8-element vector, where each element represents an action.
The action space is continuous, spanning eight components in the range of $-1.0$ to $1.0$.
These components indicate the numerical torques at hinge joints, with each action representing torque intensity and direction.

\textbf{Observation Space:}
The state space represents the ant's body, detailing the positions of its body parts followed by their corresponding velocities.
This is represented by a 27-element vector, a numerical array that captures a comprehensive view of the ant's physical state.
The observation is a vector shaped as $(27,)$.

\textbf{Reward Function:} At each time step, the reward function comprises two parts: an x-axis velocity reward $r_\text{vel} = v_x$ and a control reward $r_\text{ctrl} = \|\bm{a}\|^2$, and $r_t = r_\text{dist} + \beta_\text{ctrl} \cdot r_\text{ctrl}$. Here, $\beta_\text{ctrl}$ is a pre-configured coefficient.

\textbf{Cost Function:} At each time step, the incurred cost is equal to the value of the velocity, denoted as $ c_t = \|\bm{v}\| $.

\textbf{Constraint:} The cumulative cost on a trajectory of length $300$ for the agent should be less than $1000.0$. The cost thresholds are calculated using 50\% of the speed attained by an unconstrained PPO agent after training for sufficient samples~\citep{zhang2020first}.

\section{Experimental Details and Hyper-parameters}\label{app:exp_details}

This chapter outlines our experimental setup, covering baseline algorithm comparisons (Section~\ref{app:impl_algo}), evaluation metrics (Section~\ref{app:eval_metrics}), specific hyper-parameters, training equipment details (Section~\ref{app:hyper}), and ablation study insights on utilizing gradients from trained world models (Section~\ref{app:model_based}).

\subsection{Implementation of Baseline Algorithms}\label{app:impl_algo}

We implemented traditional safe reinforcement learning algorithms such as PDO~\citep{chow2018risk}, APPO~\citep{dai2023augmented}, CPO~\citep{achiam2017constrained}, and CUP~\citep{yang2022constrained}, strictly following their original papers.
For the training process, we refer to widely-recognized benchmarks like Safety-Gym~\citep{ray2019benchmarking} and OmniSafe~\citep{ji2023omnisafe}, achieving the optimal algorithm performance through multiple parameter adjustments.

Our discussion will focus on the BPTT and SHAC algorithms, which have been modified using the Lagrangian method to meet constraints and are named BPTT-Lag and SHAC-Lag, respectively.
These methods are also implemented based on the \gbe{} of the objective and constraint functions, thereby solving the subsequent surrogate optimization problem:
\begin{equation}\label{eq:surr_primal_lag}
\begin{aligned}
    \bm{\theta}_{k+1} = &\arg\max_{\bm{\theta}\in\Theta} \quad  \left(\bm{\theta} - \bm{\theta}_{k}\right)^\top \nabla_{\bm{\theta}} \mathcal J_R(\bm{\theta}_k)
    \\
    \text{s.t.} \quad& \mathcal J_C(\bm{\theta}_k) + \left(\bm{\theta} - \bm{\theta}_{k}\right)^\top \nabla_{\bm{\theta}} \mathcal J_C(\bm{\theta}_k) \leq b
\end{aligned}
\end{equation}
\textbf{Notations.} Considering the $ k^\text{th} $ iteration, we also introduce additional notations to make the discussion more concise:
$\bm{g}_k\triangleq\nabla_{\bm{\theta}} \mathcal J_R(\bm{\theta}_k)$,
$\bm{g}_q\triangleq\nabla_{\bm{\theta}} \mathcal J_C(\bm{\theta}_k)$,
$c_k\triangleq\mathcal{J}_C(\bm\theta_k)-b$, and
$\bm\delta\triangleq\bm\theta-\bm\theta_k$.

With these definitions, we rewrite the surrogate problem ~\eqref{eq:surr_primal_lag}:
\begin{equation}
  \max_{\bm{\delta}}~\bm{g}_k^\top\bm{\delta}\quad    \text{s.t.}~c_k + \bm{q}^\top_k\bm{\delta} \leq 0.
\end{equation}
By introducing Lagrangian multipliers $\lambda$, we construct the dual function of the above constrained problem:
\begin{equation}
L(\bm\delta,\lambda)\triangleq -\bm{g}_k^\top\bm{\delta} + \lambda\cdot(c_k + \bm{q}^\top_k\bm{\delta})
\end{equation}
By solving the following primal-dual problem,
\begin{equation}\label{eq:app_dual_function}
    \max_\lambda\min_{\bm\delta}~L(\bm\delta,\lambda)
\end{equation}
iteratively updating the policy parameters $\bm\theta$ and the dual variable $\lambda$, we can approximate the local optimal point $(\bm\theta^*, \lambda^*)$.

In this process, $\bm{g}_k=\nabla_{\bm{\theta}} \mathcal J_R(\bm{\theta}_k)$, $\bm{g}_q=\nabla_{\bm{\theta}} \mathcal J_C(\bm{\theta}_k)$ can be provided through both the BPTT and SHAC methods. Thus, we obtain two similar algorithms as follows.
\paragraph{BPTT-Lag:} We refer to the method of deriving $\bm{g}_k$ and $\bm{q}_k$ through the following loss functions
\begin{equation}
    \mathcal{L}_R(\bm{\theta}) = \frac{1}{N}\sum_{i = 1}^{N}\Big[\sum_{t = 0}^{T-1}r(\mathbf{s}^i_t, \mathbf{a}^i_t)\Big],
\end{equation}
\begin{equation}
  \mathcal{L}_C(\bm{\theta}) = \frac{1}{N}\sum_{i = 1}^{N}\Big[\sum_{t = 0}^{T-1}c(\mathbf{s}^i_t, \mathbf{a}^i_t)\Big].
\end{equation}
and iteratively solving Equation~\eqref{eq:app_dual_function} to update the policy parameters as BPTT-Lag.
\paragraph{SHAC-Lag:} We refer to the method of deriving $\bm{g}_k$ and $\bm{q}_k$ through the following loss functions
\begin{equation}
    \mathcal{L}_R(\bm{\theta}) = \frac{1}{N}\sum_{i = 1}^{N}\Big[\sum_{t = t_0}^{t_0 + h-1}r(\mathbf{s}^i_t, \mathbf{a}^i_t) + V^R_\phi(\mathbf{s}^i_{t_0 + h})\Big],
\end{equation}
\begin{equation}
  \mathcal{L}_C(\bm{\theta}) = \frac{1}{N}\sum_{i = 1}^{N}\Big[\sum_{t = t_0}^{t_0 + h-1}c(\mathbf{s}^i_t, \mathbf{a}^i_t) + V^C_\psi(\mathbf{s}^i_{t_0 + h})\Big].
\end{equation}
and iteratively solving Equation~\eqref{eq:app_dual_function} to update the policy parameters as SHAC-Lag.

\subsection{Evaluation Metrics}\label{app:eval_metrics}

In Section~\ref{sec:exp}, we apply two evaluation metrics to quantify the convergence efficiency and the satisfaction of constraints, respectively.

\subsubsection{Number of Environment Steps at Convergence (Conv. Steps)}

This evaluation metric measures the efficiency of algorithmic performance improvement by determining the number of environmental steps needed for convergence on the episodic return curve.
The convergence point on the episodic return curve is defined as the point from which the range of variation within a given window is less than a set threshold.
In Safe RL tasks, we hope the algorithm can converge to a feasible point, so it is additionally required that this convergence point meets the constraints.

% The example code for calculating this evaluation metric is as follows:
% \begin{minted}[fontsize={\small}]{python}
% def find_convergence_point(
%     env_step, returns, constraints, constraint_threshold, window_size, threshold
% ):
%     for i in range(len(returns)):
%         for j in range(i + 1, min(i + window_size, len(returns))):
%             if (
%                 (np.abs(returns[j] - returns[i]) > threshold)
%                 or (constraints[i] > constraint_threshold)
%             ):
%                 break
%         else:
%             return env_step[i]

%     return -1
% \end{minted}

\subsubsection{Violation Ratio within Safety Critical Areas (Vio. Ratio)}
Safety critical areas are defined as key regions where safety issues frequently occur~\citep{NILSSON20179083,xu2022trustworthy}.
Thus, updates nearing the constraint threshold in the policy update process are considered to occur within the safety critical areas.
Consequently, we use the proportion of constraint violations in safety critical areas during the policy update process as a measure of the capability of the update algorithm to meet constraints.
Additionally, in deep Safe RL, where safety constraints are expressed as the expected value of sampling, a soft margin is introduced to determine if the policy violates constraints.
% The example code for calculating this evaluation metric is as follows:
% \begin{minted}[fontsize={\small}]{python}
% def violation_ratio_in_critic_area(
%     constraints, constraint_threshold, critic_area_radius, soft_margin
% ):
%     num_in_critic_area = np.sum(constraints > (constraint_threshold - critic_area_radius))
%     num_violation = np.sum(constraints > (constraint_threshold + soft_margin))
%     if num_in_critic_area > 0:
%         return num_violation / num_in_critic_area
%     else:
%         return -1
% \end{minted}

% \newpage
\subsection{Hyper-parameters}\label{app:hyper}

In this section, we present the hyper-parameters for CGPO across four distinct tasks.
\begin{table}[ht]
    \label{tab:hyper_training_baselines}
    \centering
    \caption{Hyper-parameters for CGPO in different tasks.}
    \begin{tabular}{lcccc}
    \toprule
        Tasks & CartPole & Reacher & HalfCheetah & Ant \\
        \midrule
        num\_epochs & 100 & 300 & 300 & 300  \\
        num\_envs & 128 & 128 & 128 & 256  \\
        episode\_length & 300 & 300 & 200 & 300 \\
        delta\_init & 1e-3 & 1e-2 & 1e-3 &  1e-3  \\
        delta\_upper & 1e-2 & 1e-1 & 1e-2 & 1e-2 \\
        delta\_lower & 1e-4 & 1e-4 & 1e-4 & 1e-4 \\
        beta\_1 & 0.8 & 0.8 & 0.8 & 0.8 \\
        beta\_2 & 1.25 & 1.25 & 1.25 & 1.25 \\
        critic\_learning\_rate & 1e-3 & 1e-3 & 1e-3 &  1e-3  \\
        short\_horizon & 10 & 10 & 10 &  10  \\
        mini\_batch\_size & 64 & 64 & 64 &  64 \\
        max\_gradient\_norm & 1e9 & 1e9 & 1e9 &  1e9 \\
        normalize\_observations & True & True & True &  True \\
        cost\_limit & -50 & 400 & 400 &  1000 \\
        \bottomrule
    \end{tabular}
\end{table}

\newpage

\subsection{Ablation Experiments with World Model}\label{app:model_based}

We demonstrated through ablation experiments that CGPO has the potential to obtain the system's analytic gradient by training a world model, which is detailed in the description of this experiment.

\subsubsection{Implementation of Non-Differentiable Tasks}

\paragraph{Non-differentiable Function Environment.}

\begin{wrapfigure}{r}{0.4\textwidth}
\vspace{-2em}
\includegraphics[width=\linewidth]{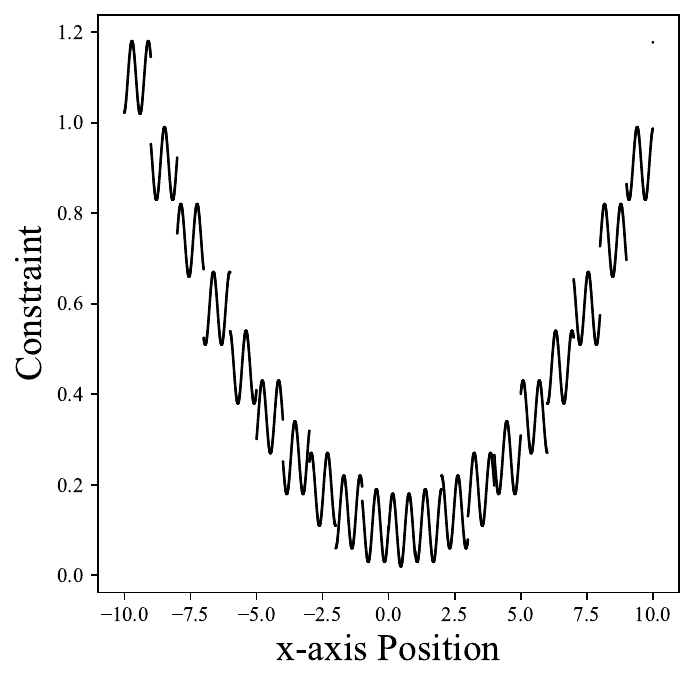}
\vspace{-2.0em}
\caption{Cost values at different x-axis positions for the agent.}
\label{fig:non-diff-func}
\end{wrapfigure}

We design a function environment similar to the simple environment in Figure~\ref{fig:simple-exp}.
The goal of this task is to train an agent to move along the x-axis, receiving rewards and incurring costs based on its position on the x-axis. The specific settings of the task are as follows:

\textbf{Action Space:} The direction and step length of the agent's movement along the x-axis constitute the action space, which is a single-element vector, i.e., $a \in (-1, 1)$.

\textbf{Observation Space:} The observation data for the agent is its position on the x-axis, which is also a single-element vector, i.e., $s \in \mathbb{R}$. The task employs a simple, artificially set state transition function: $s_{t+1} = s_t + 0.2 \times a_t$.

\textbf{Reward Function and Cost Function:} Since the purpose of this simple task is to verify the accuracy of the estimation method, we set the reward function and cost function in the same form, i.e.,
\begin{equation}
    f(x) = \left(\frac{\text{floor}(x)}{10.0}\right)^2 + 0.1 + 0.1 \times \sin\left(\frac{8x}{\pi}\right)
\end{equation}
where $\text{floor}(x)$ rounds $x$ down to the nearest integer, and its visualization is shown in Figure \ref{fig:non-diff-func}.

\textbf{Constraint:} Since the task is solely aimed at verifying the accuracy of the estimation method, the set constraint has no physical significance. Specifically, the cumulative cost on a trajectory of length 100 for the agent should be less than 8.0:
\begin{equation}
    \mathcal{J}_C(\pi) \triangleq \sum^{99}_{t=0} f(x_t)\leq 8.0
\end{equation}
This constraint setting is consistent with many well-known Safe Benchmarks, where the undiscounted cumulative cost on a finite-length trajectory must be less than a scalar threshold, namely, $\sum^{T-1}_{t=0}c_t \leq b$.

\paragraph{Non-differentiable Swimmer.} The swimmers consist of three or more segments, known as 'links', connected by a slightly smaller number of articulation joints, termed 'rotors'.
Each rotor joint smoothly connects two links, forming a linear sequence.
These swimmers elegantly move through a two-dimensional aquatic environment.
Their goal is to glide swiftly to the right, utilizing the torque applied to the rotors and effectively using the water's resistance.
The constraint involves requiring the velocity to be less than a threshold, as referenced from~\citet{zhang2020first}.
This task has poor differentiability because the differentiable kernel of Brax is not used.

\textbf{Action Space:} The action space is defined as a bounded box, designated as Box(-1, 1, (2,), float32).
The agent takes a 2-element vector for actions.
The values for these actions fall within the range of $-1.0$ to $1.0$.
It represents the torque applied on the rotor.

\textbf{Observation Space:}
The state space is composed of a series of elements that define the position and movement dynamics within the environment.
Specifically, these elements provide detailed information about the positions, angles, and their respective velocities and angular velocities of the segments in our system.
The observation is a vector shaped as $(8,)$.

\textbf{Reward Function:} At each time step, the reward function comprises two parts: an x-axis velocity reward $r_\text{vel} = v_x$ and a control reward $r_\text{ctrl} = \|\bm{a}\|^2$, and $r_t = r_\text{dist} + \beta_\text{ctrl} \cdot r_\text{ctrl}$. Here, $\beta_\text{ctrl}$ is a pre-configured coefficient.

\textbf{Cost Function:} At each time step, the incurred cost is equal to the value of the velocity, denoted as $ c_t = \|\bm{v}\| $.

\textbf{Constraint:} The cumulative cost on a trajectory of length $300$ for the agent should be less than $150.0$. The cost thresholds are calculated using 50\% of the speed attained by an unconstrained PPO agent after training for sufficient samples~\citep{zhang2020first}.

\subsubsection{Basic Model-based CGPO algorithm}

Here, we present a basic implementation of using a World Model to provide analytic gradients for CGPO, which we refer to as MB-CGPO. The input to this world model is $s_t$ and $a_t$, and the output is a combined vector of $s_{t+1}$, $r_t$, and $c_t$. By using the MSE loss, we train the world model $\mathcal{F}$ with collected data, as shown in:
\begin{equation}\label{eq:wm_loss}
    L(s_t, s_{t+1}, r_t, c_t) = \left(\mathcal{F}_s(s_t, a_t) - s_{t+1}\right)^2 + \left(\mathcal{F}_r(s_t, a_t) - r_{t}\right)^2 + \left(\mathcal{F}_c(s_t, a_t) - c_{t}\right)^2
\end{equation}
When computing analytic gradients, we collect trajectories through the World Model and then derive along the trajectory in the exact same manner as with a differentiable simulator. Thus, the MB-CGPO is given in Algorithm~\ref{alg:mb-cgpo}.

\begin{algorithm}[H]
  \caption{Model-based Constrained Gradient-based Policy Optimization (MB-CGPO)}
  \label{alg:mb-cgpo}
  \begin{algorithmic}
    \STATE \textbf{Input:} Initialize World Model $\mathcal{F}_{\varphi_0}$, policy $\pi_{\bm{\theta}_0}$, critic $V_{\bm{\phi}_0}$ and $V^c_{\bm{\psi}_0}$, radius $\hat\delta_0$, and number of iterations $K$.
    \FOR {$k=1,2,\ldots,K$}
    \STATE Sample a set of trajectories $\mathcal D = \{\tau\} \sim \pi_{\bm{\theta}_k}$.
    \STATE Update World Model $\mathcal{F}_{\varphi_k}$ with data $\mathcal{D}$ using~\eqref{eq:wm_loss}.
    \STATE Sample a set of trajectories $\mathcal D_M = \{\tau\} \sim \pi_{\bm{\theta}_k}$ in World Model $\mathcal{F}_{\varphi_{k+1}}$.
    \STATE Compute the $\bm{g}_k$, $\bm{q}_k$ using \eqref{eq:loss_r} and \eqref{eq:loss_c} with $\mathcal{D}_M$.
    \IF {$c_k^2/\bm{q}^\top_k\bm{q}_k -\hat\delta_k\geq0$ and $c_k > 0$}
    \STATE Update the policy as $\bm{\theta}_{k+1} = \bm{\theta}_{k} - \frac{\sqrt{\hat\delta_k}}{\|\bm{q}_k\|}\bm{q}_k$.
    \ELSIF {$c_k^2/\bm{q}^\top_k\bm{q}_k-\hat\delta_k\geq0$ and $c_k<0$}
    \STATE Update the policy as $\bm{\theta}_{k+1} = \bm{\theta}_{k} + \frac{\sqrt{\hat\delta_k}}{\|\bm{g}_k\|}\bm{g}_k$.
    \ELSE
    \STATE Compute dual variables $\lambda^*_k$, $\nu^*_k$ using Algorithm~\ref{alg:calc_dual_var}.
    \STATE Update the policy as $\bm{\theta}_{k+1} = \bm{\theta}_{k} + \frac{\bm{g}_k-\nu^*_k\bm{q}_k}{\lambda^*_k}$.
    \ENDIF
    \STATE Update $V_{\bm{\phi}_k}$, $V^c_{\bm{\psi}_k}$ using \eqref{eq:value}, and $\hat{\delta}_{k+1}$ using \eqref{eq:update_radius}.
    \ENDFOR
    \STATE \textbf{Output:} Policy $\pi_{\bm{\theta}_K}$.
  \end{algorithmic}
\end{algorithm}

\end{document}